%% file: mimir.tex
\documentclass[conference]{IEEEtran}

\ifCLASSINFOpdf
\else
\fi

\usepackage{amsmath,amssymb,amsfonts}
\usepackage{algorithmic}
\usepackage{graphicx}
\usepackage{textcomp}
\usepackage{url}
\usepackage{paralist}
\usepackage{multirow}
\usepackage{mathtools}
\usepackage{algorithm}
\usepackage{amsthm}
\usepackage{booktabs}
\usepackage{bbm}
\usepackage[table]{xcolor}
\usepackage{todonotes}
\usepackage{arydshln}
\usepackage{tablefootnote}
\usepackage[colorlinks=true, linkcolor=blue, citecolor=green, urlcolor=red]{hyperref}
\usepackage[available,functional,reproduced]{ndssbadges}

\newcommand{\tabincell}[2]{\begin{tabular}{@{}#1@{}}#2\end{tabular}}

\newtheorem{thm}{Theorem}[section]
\numberwithin{thm}{section}
\newtheorem{lem}[thm]{Lemma}

\newtheorem{proposition}[thm]{Proposition}

\DeclarePairedDelimiterX{\set}[1]{\{}{\}}{\setargs{#1}}
\NewDocumentCommand{\setargs}{>{\SplitArgument{1}{;}}m}
{\setargsaux#1}
\NewDocumentCommand{\setargsaux}{mm}
{\IfNoValueTF{#2}{#1} {#1\,\delimsize|\,\mathopen{}#2}}

\definecolor{Light}{rgb}{0.99, 0.92, 0.95}
\definecolor{deemph}{gray}{0.0}
\definecolor{bblue}{rgb}{1,1,1}
\definecolor{upcolor}{RGB}{57,182,74}
\newcommand{\up}[1]{\small \textcolor{upcolor}{$\uparrow$#1}}

\begin{document}
\title{MIMIR: Masked Image Modeling for Mutual Information-based Adversarial Robustness}

\author{\IEEEauthorblockN{Xiaoyun Xu$^{1}$,
Shujian Yu$^{2}$,
Zhuoran Liu$^{1}$ 
and Stjepan Picek$^{3,1}$}
\IEEEauthorblockA{$^{1}$Radboud University Nijmegen, The Netherlands}
\IEEEauthorblockA{$^{2}$Vrije Universiteit Amsterdam, The Netherlands}
\IEEEauthorblockA{$^{3}$University of Zagreb Faculty of Electrical Engineering and Computing, Croatia\\
Email: xiaoyun.xu@ru.nl, s.yu3@vu.nl, z.liu@cs.ru.nl, stjepan.picek@ru.nl}}

\IEEEoverridecommandlockouts
\makeatletter\def\@IEEEpubidpullup{6.5\baselineskip}\makeatother
\IEEEpubid{\parbox{\columnwidth}{
		Network and Distributed System Security (NDSS) Symposium 2026\\
		23 - 27 February 2026 , San Diego, CA, USA\\
		ISBN 979-8-9919276-8-0\\  
		https://dx.doi.org/10.14722/ndss.2026.241813\\
		www.ndss-symposium.org
}
\hspace{\columnsep}\makebox[\columnwidth]{}}

\maketitle

\begin{abstract}
Vision Transformers (ViTs) have emerged as a fundamental architecture and serve as the backbone of modern vision-language models. Despite their impressive performance, ViTs exhibit notable vulnerability to evasion attacks, necessitating the development of specialized Adversarial Training (AT) strategies tailored to their unique architecture. 
While a direct solution might involve applying existing AT methods to ViTs, our analysis reveals significant incompatibilities, particularly with state-of-the-art (SOTA) approaches such as Generalist~\cite{Wang2023Generalist} (CVPR 2023) and DBAT~\cite{levi2024dbat} (USENIX Security 2024).
This paper presents a systematic investigation of adversarial robustness in ViTs and provides a novel theoretical Mutual Information (MI) analysis in its autoencoder-based self-supervised pre-training.
Specifically, we show that MI between the adversarial example and its latent representation in ViT-based autoencoders should be constrained via derived MI bounds.
Building on this insight, we propose a self-supervised AT method, MIMIR, that employs an MI penalty to facilitate adversarial pre-training by masked image modeling with autoencoders.
Extensive experiments on CIFAR-10, Tiny-ImageNet, and ImageNet-1K show that MIMIR can consistently provide improved natural and robust accuracy, where MIMIR outperforms SOTA AT results on ImageNet-1K. 
Notably, MIMIR demonstrates superior robustness against unforeseen attacks and common corruption data and can also withstand adaptive attacks where the adversary possesses full knowledge of the defense mechanism.
Our code and trained models are publicly available at: \url{https://github.com/xiaoyunxxy/MIMIR}.
\end{abstract}

\IEEEpeerreviewmaketitle

\section{Introduction}
\label{sec:intro}

ViTs~\cite{alexey2021vit} and their variants~\cite{Liu_2021_ICCV,pmlr-v139-d-ascoli21a} have achieved substantial progress and serve as foundational components in modern vision-language models. Prominent multimodal frameworks, including CLIP~\cite{radford2021clip}, BLIP~\cite{li2022blip}, and Mini-GPT4~\cite{zhu2024minigpt}, typically employ ViTs as their image encoders.
However, similar to convolutional neural networks (CNNs), attention-based models provide limited robustness against evasion attacks~\cite{NEURIPS2021_e19347e1, aldahdooh2021reveal, NEURIPS2022_760b5def, 10.1007/978-3-031-19778-9_18, 10.1007/978-3-031-19775-8_24}.
Evasion attacks~\cite{10.1007/978-3-642-40994-3_25,szegedy2013intriguing} (also known as adversarial attacks), where well-trained deep models are fooled by introducing human-imperceptible perturbations to inputs, remain a persistent challenge in deep learning security.
In 2024, the National Institute of Standards and Technology (NIST) explicitly listed adversarial attacks as a significant threat to AI systems, and pointed out the importance of conducting robustness testing and mitigation, such as adversarial training and formal verification, when deploying AI tools~\cite{nist2024adversarialml}.
Nevertheless, improving adversarial robustness remains a difficult task, where even
SOTA methods, such as~\cite{NEURIPS2023_2d3b0076,peng2023robarch,bai2024mixednuts}, achieve only marginal robustness gains, commonly below 2\% compared with those before them.

So far, Adversarial Training (AT) is widely recognized as the most practically effective defense~\cite{10.1007/978-3-031-19778-9_18,NEURIPS2022_760b5def,10136149} against evasion attacks.
AT operates by augmenting the training dataset with adversarially perturbed samples~\cite{DBLP:conf/iclr/MadryMSTV18}, yet introduces two key limitations: (1) substantial computational overhead due to the generation of adversarial examples during training~\cite{DBLP:conf/iclr/MadryMSTV18}, and (2) a potential degradation in natural accuracy~\cite{raghunathan2019adversarial}.
Numerous methods have been proposed to mitigate these challenges.
Techniques such as FreeAT~\cite{NEURIPS2019_7503cfac} optimize efficiency by reusing gradient information during adversarial example generation, while FastAT~\cite{Wong2020Fast} employs an improved Fast Gradient Sign Method (FGSM) to accelerate training.
TRADES~\cite{zhang2019trades}, SCORE~\cite{pang2022robustness}, Generalist~\cite{Wang2023Generalist}, and DBAT~\cite{levi2024dbat} explore how to achieve the best trade-off between natural and robust accuracy.
Additionally, pre-training strategies have also been leveraged to enhance the performance of AT~\cite{NEURIPS2019_a2b15837,NEURIPS2023_2d3b0076}.

Applying existing AT methods to ViTs presents unique challenges due to the fundamental differences between attention-based architectures and CNNs. 
Unlike CNNs, ViTs lack inductive biases~\cite{alexey2021vit}, including locality, two-dimensional neighborhood structure, and translation equivariance. 
These biases are inherent to CNNs as prior knowledge, enabling efficient learning with limited data, whereas ViTs typically require larger training datasets to achieve comparable generalization performance~\cite{alexey2021vit}.
Consequently, AT for ViTs entails significantly higher computational costs.
Initial research on AT for ViTs explored their unique attention mechanism.
For instance, robustness can be improved by randomly dropping gradients according to attention~\cite{NEURIPS2022_760b5def} or improving training efficiency by dropping low-attention image embeddings~\cite{10.1007/978-3-031-19778-9_18}.
The majority of recent works have focused on adapting CNN-based AT methodologies to ViTs, given AT's success in building robust CNNs.
Unfortunately, standard CNN AT techniques are not fully transferable to ViTs.
Empirical studies~\cite{NEURIPS2021_e19347e1,10136149} reveal that \emph{strong data augmentations} (such as Randaugment~\cite{Cubuk_2020_CVPR_Workshops}, CutMix~\cite{Yun_2019_ICCV}, and MixUp~\cite{DBLP:conf/iclr/ZhangCDL18}, which improve robustness in CNNs) often degrade AT performance for ViTs.
To mitigate this, recent work~\cite{NEURIPS2021_e19347e1} suggests progressively increasing augmentation intensity (e.g., distortion magnitudes in RandAugment or the sampling probability of MixUp/CutMix) during training.
Furthermore, SOTA AT strategies, such as Generalist~\cite{Wang2023Generalist} and DBAT~\cite{levi2024dbat}, are less effective when applied to ViTs (see Table~\ref{tab:cifar_at}), and there is a lack of evaluation on large datasets, such as ImageNet-1K, which further limits their generalizability.

Pre-training has emerged as a complementary approach to ViT AT, with studies showing that adversarial fine-tuning of naturally pre-trained models can enhance robustness~\cite{NEURIPS2022_760b5def,NEURIPS2023_2d3b0076}. AdvXL~\cite{wang2024revisiting} notably advanced this paradigm by developing efficient AT for web-scale datasets. However, the mechanisms underlying pre-training's effectiveness are not fully understood, and results are inconsistent across implementations. For instance, models pre-trained on ImageNet-21K using SimMIM~\cite{Xie_2022_CVPR} demonstrate comparable performance to scratch-trained counterparts, while CLIP~\cite{radford2021clip} pre-training has been shown to degrade performance in some configurations~\cite{liu2024comprehensive}.

While previous methods of ViT AT, such as~\cite{NEURIPS2023_2d3b0076,10136149,NEURIPS2021_e19347e1}, focus on searching for better combinations of training hyperparameters, they suffer from performance drops across different architectures and datasets. In contrast, we aim for a generalizable method via pre-training.
Specifically, this work presents a systematic investigation of self-supervised pre-training for ViT robustness through the lens of Mutual Information (MI) and Information Bottleneck (IB) theory.
IB introduces a joint objective of simultaneously minimizing the MI between inputs and latent features while maximizing the MI between labels and latent features to mitigate the impact of the adversarial noise in the inputs.
Regarding the ViT AT, we develop a novel theoretical justification for self-supervised autoencoders, demonstrating that reducing MI between inputs and latent features enhances ViT robustness.
Based on this finding, we propose a theoretically grounded adversarial pre-training method, \textbf{M}asked \textbf{I}mage \textbf{M}odeling for Mutual \textbf{I}nformation-based Adversarial \textbf{R}obustness (\textbf{MIMIR}).
\footnote{Mimir is a figure in Norse mythology, renowned for his knowledge and wisdom.}
Specifically, we convert Masked Image Modeling (MIM) into an effective and efficient adversarial pre-training method.
The basic idea is to predict the masked content of inputs, which is a self-supervised learning task.
The effectiveness of MIMIR is analyzed and guaranteed by our theoretical justification.
The efficiency comes from discarding the masked content (75\% of image patches are discarded in our experiments), which greatly reduces the computing requirements.
Figure~\ref{fig:method} provides an illustrative diagram of MIMIR.

We validate MIMIR's effectiveness through extensive experiments on CIFAR-10~\cite{krizhevsky2009learning}, Tiny-ImageNet~\cite{le2015tiny}, and ImageNet-1K~\cite{5206848}, showing consistent improvements in both natural and adversarial accuracy.
In addition, we test the generalizability of MIMIR by combining MIMIR with three MIM methods for three representative architectures, including MAE~\cite{He_2022_CVPR} for ViTs, Group Window Attention~\cite{huang2022green} for hierarchical transformer (Swin~\cite{Liu2021swin}), and SparK~\cite{tian2023spark} for CNN (ConvNext~\cite{liu2022convnext}), where MIMIR outperforms SOTA AT methods on ImageNet-1K.
Our main contributions are:
\begin{compactitem}
    \item By revisiting the current ViT AT strategies, we point out that ViT AT methods compromise natural accuracy and lack systematic study. To this end, we provide a theoretical analysis using adversarial examples and the Information Bottleneck. We theoretically show that the Mutual Information between adversarial examples and the learned latent representation should be decreased for better robustness.
    \item Based on our analysis, we propose a self-supervised defense, MIMIR, against adversarial attacks on ViTs.
    We evaluate MIMIR using multiple architectures on three datasets under various adversarial attacks, demonstrating its effectiveness.
    MIMIR achieves SOTA adversarial robustness on ImageNet-1K following the standardized evaluation by RobustBench.\footnote{\url{https://robustbench.github.io/}}
    \item We show that MIMIR is robust against unforeseen attacks and common corrupted data (ImageNet-C~\cite{hendrycks2019robustness}) and can resist adaptive attacks where the adversary is aware of MIMIR's design. 
    
\end{compactitem}

\begin{figure}[t]
\centering
\includegraphics[width = 1\linewidth]{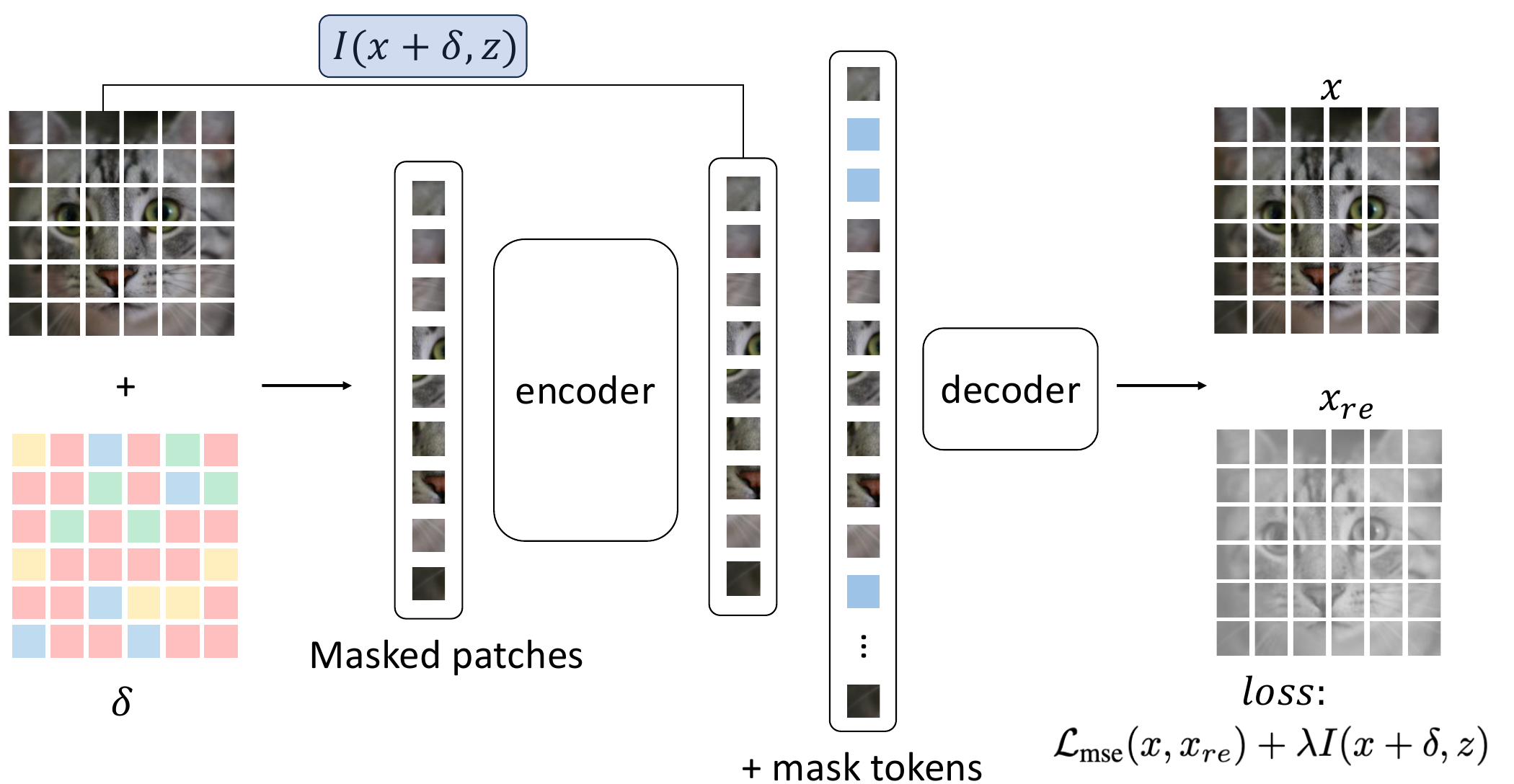}
\caption{Diagram illustrating the working mechanism of MIMIR. 
In the pre-training, adversarial images $x + \delta$ are generated and separated into image patches as the inputs of the encoder. The output of decoder $x_{re}$ and the natural input image $x$ are used to calculate the loss. 
After pre-training, a trained encoder is combined with a randomly initialized classification layer as the final complete model. Then the complete model is further fine-tuned for classification tasks.
}
\label{fig:method}
\end{figure}

\section{Background}
\label{sec:background}

\subsection{Evasion Attack}
Evasion attacks~\cite{10.1145/1014052.1014066,10.1007/978-3-642-40994-3_25,szegedy2013intriguing}, also known as adversarial attacks, refer to applying imperceptible perturbations to the original input of the machine learning model (in this work, neural network), which generates adversarial examples to fool the victim model.
Given an $L$-layer neural network $F_\theta$ for classification in $d_Y$-dimensional space and a training dataset $D = \set{(x_i, y_i)}_{i=1}^{n}$ in $d_X$-dimensional space, the two primary goals of adversarial attacks are:
\begin{compactenum}
    \item The generated perturbation $\delta$ can successfully mislead the network by maximizing (e.g., PGD attack~\cite{DBLP:conf/iclr/MadryMSTV18}): 
\begin{equation}\label{eq:pgdmax}
\mathop{\max}\limits_{\delta \in S}\mathcal{L}_{CE}(\theta, x_i+\delta, y_i),
\end{equation}
where $x_i \in \mathbb{R}^{d_X}$ and $y_i \in \set{0,1}^{d_Y}$, 
$\theta$ are the parameters of the current network and $\mathcal{L}_{CE}$ is the standard CE (Cross-Entropy) loss. The perturbation aims to decrease confidence in the ground truth labels while increasing confidence in wrong labels. Thus, the loss between the misleading outputs and the ground truth labels increases. 
\item The generated adversarial examples are as similar as possible to the original clean examples by limiting $\delta$ to a relatively small domain:
\begin{equation}
\label{eq:ballfunction}
S = B(x_i, r) = \set{ \delta \in \mathbb{R}^{d_X} : \left \| \delta \right\|_\infty \leq r },
\end{equation}
where $S$ is the $l_\infty$-ball of radius $r$ at a position $x_i$ in $\mathbb{R}^{d_X}$.
The distance between $x_i$ and $x_i+\delta$ can be evaluated by norms such as $l_2$ and $l_\infty$.
\end{compactenum}

When using the above attacks to generate adversarial examples for AT, the learning objective is:
\begin{equation}\label{eq:adversarialtrainingpgdminmax}
	\mathop{\min}\limits_{\theta} \mathop{\max}\limits_{\delta \in S} \mathcal{L}_{CE}(\theta, x_i+\delta, y_i).
\end{equation}

\subsection{Masked Image Modeling - MIM}
MIM refers to a self-supervised pre-training framework that aims to reconstruct pre-defined targets, such as discrete tokens~\cite{bao2022beit}, raw RGB pixels~\cite{He_2022_CVPR,Li2022ummae}, or features~\cite{Wei_2022_CVPR}.
The final goal is to use the pre-trained model as a starting point for downstream fine-tuning. 
The downstream tasks include, for instance, classification and object detection.
More specifically, to build a high-performance ViT $f_e$ without a classification layer, we consider $f_e$ as an encoder to extract discriminative input features.
Then, we design a lightweight decoder $f_d$, which uses the output of $f_e$ as its input.
The goal of $f_d$ is to reconstruct the original inputs (let us consider MAE~\cite{He_2022_CVPR} as an example).
The aim is to decrease the distance between the input $x$ and $x_{re} = f_d \circ f_e(x)$.
After the encoder $f_e$ and decoder $f_d$ are trained, we use $f_e$ plus a manually initialized classification layer as the starting point of fine-tuning.

\subsection{Mutual Information - MI}
MI measures the mutual dependence between two random variables, $X$ and $Y$. 
It quantifies the amount of information contained in one random variable about another random variable or the reduced uncertainty of a random variable when another random variable is known.
It can be written as:
\begin{equation}\label{eq:mi}
	I(X, Y) = \int_{\mathcal{Y}}\int_{\mathcal{X}}P_{(X,Y)}(x, y) \log \left( \frac{P_{(X,Y)}(x, y)}{P_{(X)}(x)P_{(Y)}(y)} \right),
\end{equation}
where $P_{(X,Y)}$ is the joint probability density function of $X$ and $Y$.
$P_{(X)}$ and $P_{(Y)}$ are the marginal probability density function of $X$ and $Y$.
MI can be equivalently expressed as:
\begin{equation}
    \begin{aligned}
        I(X, Y) &= H(X) - H(X|Y).
    \end{aligned}
\end{equation}

Estimating the exact MI is not easy, as it is difficult to precisely estimate $P_{X, Y}$ or $P_{X}$ and $P_{Y}$ in high-dimensional space~\cite{9082644}.
In practice, Deep Deterministic Information Bottleneck (DIB)~\cite{9414151} suggested using the matrix-based Renyi's $\alpha$-order entropy $I_{\alpha}$~\cite{6954500,8787866} to estimate MI, which avoids density estimation and variational approximation.
An alternative way is the Hilbert-Schmidt Independence Criterion (HSIC)~\cite{10.1007/11564089_7}, which is a kernel-based dependence measure defined in a reproducing kernel Hilbert space (RKHS) and is usually used as a surrogate of MI.
Details about definitions and empirical estimators of $I_{\alpha}$ and HSIC are provided in Appendix~\ref{sec:mi}. 
In this paper, we evaluate both $I_{\alpha}$ and HSIC as our MI measurements.

\subsection{Information Bottleneck - IB}
The IB concepts were first proposed in~\cite{DBLP:journals/corr/physics-0004057} and further developed for deep learning in~\cite{7133169, DBLP:journals/corr/Shwartz-ZivT17}. 
IB describes the generalization of a deep network in two phases:
1) empirical error minimization (ERM) and 2) representation compression~\cite{DBLP:journals/corr/Shwartz-ZivT17}.
For a network with input $x$ and label $y$, there is an intermediate representation $t_l$ for each layer $l$, i.e., the output of the $l$-th layer.
The IB principle aims to keep more relevant information in $t_l$ about the target $y$ while decreasing the irrelevant information about the input $x$. 
The information between intermediate representation $t_l$ and input $x$ or label $y$ is quantified by MI, denoted by $I(\cdot)$.
During neural network training, in the ERM phase, the model increases shared information between $t_l$ with respect to both $x$ and $y$. 
Afterward, in the compression phase, the model decreases information contained in $t_l$ about $x$ but preserves (or even increases) information about $y$.
The reduction of $I(x,t_l)$ can be interpreted as a way of reducing noise or compressing irrelevant or redundant features in $x$. At the end of the training, the model strikes a trade-off that maximizes $I(y,t_l)$ and minimizes $I(x,t_l)$.   
Formally, the IB minimizes the following Lagrangian:
\begin{equation}\label{eq:ib_obj}
	\mathcal{L} = I(x,t_l) - \beta I(y,t_l),
\end{equation}
where $\beta$ is a Lagrange multiplier that controls the trade-off between predicting $y$ and compressing $x$.

\section{MIMIR}
\label{sec:method}

\subsection{Threat Model}

\noindent
\textbf{Adversary's goal.} 
The attacker aims to fool the trained model with both non-targeted and targeted attacks. 
The goal is to decrease the overall classification accuracy (non-targeted) or compel the model to recognize any inputs as a specific target (targeted). 
Meanwhile, the adversarial perturbations applied to the input should be invisible so that they will not be easily detected. 
During the training phase, the model optimizes its parameters to minimize the loss between predicted outputs and true labels, thereby enhancing classification accuracy. In contrast, the adversary's objective is to develop algorithms that generate perturbations capable of maximizing this loss.
For a targeted attack, the attacker decreases the loss between the output and the specified target label. 
To maintain the imperceptibility of the perturbations, the magnitude of the adversarial modifications is constrained by distance measures (such as $l2$ and $l_\infty$), ensuring that the alterations to the input data remain within a visually indistinguishable range.

\noindent
\textbf{Adversary knowledge.} 
The attacker has white-box access to the model, including training data, architectures, hyperparameters, and model weights. 
The attacker can implement iterative attacks and unlimited queries to update adversarial examples multiple times in white and black-box settings. 
Adversarial examples can be created according to model architectures, parameters, the gradients of the loss function, and datasets. 
In addition, we also consider adaptive adversaries who are also aware of potential defenses. 
The adversary can design new attacks for a specific model according to the design details of the defense method.

\noindent
\textbf{Defender's goal.}
From the defender's perspective, the main goal is to train a robust model against potential adversarial attacks.
The defender considers the following objectives:
\begin{compactitem}
    \item The defender aims to prevent the performance of natural data from decreasing significantly, but allows a slight drop of natural accuracy for a trade-off in exchange for robustness.
    \item The defense method should provide a notable improvement compared to models without defenses when subjected to various adversarial attacks, especially to adaptive attacks that are aware of the details of the defense method.
    \item The defense method should be efficient and scalable to large datasets such as ImageNet-1K~\cite{5206848}.
\end{compactitem}

\subsection{Design Intuition}
\label{sec:designintuition}
MIM has been established as an effective pre-training paradigm for Vision Transformers (ViTs), demonstrating strong performance across diverse downstream tasks~\cite{vincent2010stacked,He_2022_CVPR,Wei_2022_CVPR,bao2022beit}. 
The core methodology involves masking foreground regions of input images and tasking the model with their reconstruction.
Masking foreground (as opposed to background) regions removes high-information content and results in a harder task (than reconstructing background content), forcing the model to develop stronger feature representations to reconstruct semantically meaningful patterns~\cite{Wang_2023_CVPR}.

Building upon these principles, we introduce an adversarial extension to MIM by incorporating adversarial perturbations into the input space. 
Our formulation is grounded in three interconnected hypotheses:
\begin{compactitem}
    \item Adversarial Robustness through Reconstruction: If a model can reconstruct natural images from adversarially perturbed inputs, its latent representations must inherently discard perturbation-specific information while preserving natural data semantics.
    \item IB Perspective: The encoder-decoder architecture naturally imposes an information bottleneck. When processing adversarial examples $x+\delta$, the system must suppress perturbation-derived information $(\delta)$ while retaining natural data information $(x)$ to achieve accurate reconstruction (see Figure~\ref{fig:method}).
    \item Optimal Masking Strategy: Complete foreground masking (to build a difficult task) is suboptimal, as it eliminates essential reconstruction signals. Instead, our method employs partial masking to maintain a tractable yet challenging learning objective.
\end{compactitem}

\renewcommand{\algorithmicrequire}{\textbf{Input:}}
\renewcommand{\algorithmicensure}{\textbf{Output:}}
\begin{algorithm}[t]
	\caption{MIMIR Pre-training} 
	\label{alg:pretrainingprocess} 
	\begin{algorithmic}[1]
		\REQUIRE{training data $D$, number of epochs $E$, encoder $f_e$, decoder $f_d$, network parameters $\theta$, $\mathcal{L}_{\text{mse}}$}, $\lambda$.
		\ENSURE{optimized weights $\theta$}
		\FOR{$e=0 \to E-1$}
		\STATE $x \leftarrow$ \texttt{sample\_batch($D$)}
            \STATE $\delta \leftarrow $ \texttt{random\_initialization}
            \STATE $x_{re} \leftarrow f_d \circ f_e(x+\delta)$
            \STATE $\delta \leftarrow \mathop{max}\limits_{\delta \in S}\mathcal{L}_{\text{mse}}(x+\delta, x_{re})$
            \STATE Forward:
		\STATE $z \leftarrow f_e (x+\delta) $ 
            \STATE $x_{re} \leftarrow f_d(z)$
            \STATE $loss \leftarrow \mathcal{L}_{\text{mse}} (x, x_{re}) + \lambda I(x+\delta, z)$
            \STATE Backward: 
            \STATE $\theta \leftarrow \theta - \alpha \nabla loss$
		\ENDFOR 
	\end{algorithmic} 
\end{algorithm}

\subsection{Design Details}

\noindent
\textbf{Autoencoder.}
MIMIR consists of an encoder $f_e$ and a decoder $f_d$ aligned with the general design of MAE~\cite{He_2022_CVPR}. 
As with other autoencoders, the encoder extracts discriminative features $z$ from inputs $x$.
The decoder reconstructs the original inputs according to the discriminative features. 
Following the design of ViT~\cite{alexey2021vit}, the input $x$ is separated into non-overlapping image patches.
We randomly mask out a part of the patches and use the remaining patches as inputs for the following process in the encoder.
This random masking process uses uniform distribution to prevent potential sampling bias, such as all foreground information being masked, as it becomes infeasible to find the reconstruction target.
Thus, we aim to keep a part of the foreground as a hint for reconstruction.
The information of masked content is recorded as mask tokens $m$, which are not used by the encoder but reserved for later use by the decoder.
Each token is a learned vector indicating the presence of a masked patch to be predicted.
The mask token is shared by all inputs of the decoder.
Like unmasked patches, mask tokens are also assigned corresponding positional embeddings to be in the correct location in the reconstructed image.
We emphasize that mask tokens are not used in the encoder part.

To train a ViT, we use the same transformer blocks as ViT to build the encoder. 
The encoder only processes the visible patches, making training much more efficient.
When converting to other architectures, such as ConViT~\cite{pmlr-v139-d-ascoli21a}, we use corresponding transformer blocks to build the encoder.
The decoder accepts the encoded visible image patches and mask tokens as inputs.
The decoder is built using the same transformer blocks as the encoder instead of using ViT~\cite{alexey2021vit} transformer blocks for all.
Then, the decoder is followed by a fully connected layer, which outputs the same number of patches as the original image.

\noindent
\textbf{Adversarial Pre-training Target.}
The training target is to extract discriminative features from visible image patches by the encoder and then reconstruct the invisible patches by the decoder.
Therefore, we need a differentiable measurement to quantify the distance between the original image and the reconstructed results.
Following the original MAE~\cite{He_2022_CVPR}, this distance is measured by Mean Squared Error (MSE).
To create a more difficult reconstruction task, we apply adversarial perturbations $\delta$ to the inputs of the encoder.
Thus, the adversarial perturbations are also masked along with the image upon input into the encoder.
The decoder reconstructs the original natural inputs by using the latent features $z$ extracted from adversarial examples.
The outputs of decoder $x_{re}$ and $x$ are used to calculate the MSE loss, which is further used to optimize the model.
Note that the reconstruction differs from the original MAE; we do not use the encoder inputs as reconstruction targets.
Formally, the pre-training process (described in Algorithm~\ref{alg:pretrainingprocess}) can be written as follows:
\begin{equation}
\label{eq:pretrain}
    \begin{aligned}
        z = f_e(x+\delta), ~
         x_{re}= f_d(z), \\
        loss_{\text{mse}} = \mathcal{L}_{\text{mse}} (x, x_{re}).
    \end{aligned}
\end{equation}

\noindent
\textbf{MI as Penalty.}
Inspired by IB, we show in Section~\ref{sec:theomotiv} that MI between latent representation and adversarial examples decreases as the accuracy on adversarial examples, i.e., $I(x+\delta, z)$, decreases while training.
Motivated by this finding, we directly use $I(x+\delta, z)$ as a penalty in our final loss function:
\begin{equation}
\label{eq:mi_loss}
    \begin{aligned}
        loss_{\text{mi}} &= \mathcal{L}_{\text{mse}} (x, x_{re}) + \lambda I(x+\delta, z),
    \end{aligned}
\end{equation}
where $\lambda$ is a regularizer for the MI penalty.
We use $I(x+\delta, z)$ instead of $I(x, z)$ as a penalty.
This is because $x \to x+\delta \to z$ follows the Markov chain since $z$ is extracted from $x+\delta$.
According to Data Processing Inequality (DPI)~\cite{beaudry2011intuitive}, $I(x, z) \leq I(x+\delta, z)$.
$I(x+\delta, z)$ is closer to $z$ on the Markov chain.

\noindent
\textbf{Generating Adversarial Examples.}
To conduct the adversarial pre-training, we need an attack that finds proper adversarial perturbations $\delta$.
As the autoencoder does not provide classification outputs, it is not possible to directly use existing adversarial attacks, such as PGD~\cite{DBLP:conf/iclr/MadryMSTV18}.
Nevertheless, it is feasible to design a new algorithm to find $\delta$ by maximizing $loss_{\text{mse}}$ in Eq.~\eqref{eq:pretrain}.
As the feature $z$ is extracted from only visible image patches, we only attack the visible patches.
We do not add any perturbations to mask tokens since the outputs of the autoencoder are only impacted by visible patches.
Then, the adversarial pre-training learning objective can be written as:
\begin{equation}
\label{eq:advobjective}
	\begin{split}
	 \mathcal{L}_{\text{adv}} = \mathop{\max}\limits_{\delta \in S} & \mathcal{L}_{\text{mse}}(f_d \circ f_e(x+\delta), x), \\
	\mathop{\min}\limits_{\theta} \mathcal{L}_{\text{adv}} & + \lambda I(x+\delta, z),
	\end{split}
\end{equation}
where $\theta$ are the autoencoder parameters.
After the autoencoder is trained, we discard the decoder and initialize a classification layer for the encoder to build a complete model.
Finally, the complete model is fine-tuned by AT methods.

\begin{figure}[t]
\centering
\includegraphics[width = 0.75\linewidth]{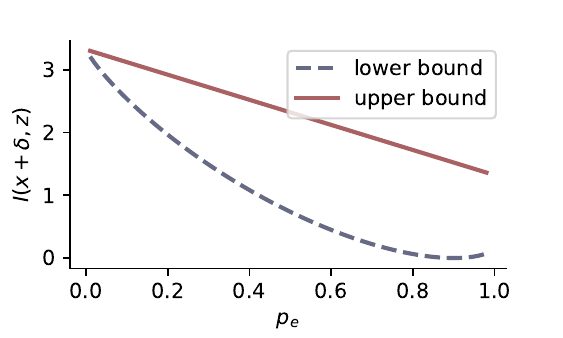}
\caption{The example plots for the lower and upper bounds on the MI in Propositions~\ref{prop:one} and~\ref{prop:two}. The entropy ($H(\cdot)$) is chosen uniformly at random from a set of 10 classes. The lower bound reaches its minimum at $p_e=0.9$.}
\label{fig:bounds}
\end{figure}

\subsection{Theoretical Justification}
\label{sec:theomotiv}
Next, we provide theoretical justification showing that MI between the adversarial example and its latent representation, i.e., $I(x+\delta,z)$, should be constrained. Let $F$ denote any classifier trained on natural samples with desirable prediction accuracy, which may suffer from adversarial attacks. We begin our analysis by first presenting Lemma~\ref{lemma1}.

\begin{lem}\label{lemma1}
Let $F(x+\delta)$ and $F(x_{re})$ denote, respectively, the predicted labels of adversarial sample $x+\delta$ and reconstructed sample $x_{re}$, we have: 
\begin{equation}
    \label{eq:assumption_mi}
    I(F(x+\delta), F(x_{re})) \leq I(F(x+\delta),x_{re}) \leq I(x+\delta, z).
\end{equation}
\end{lem}

\begin{proof}
There are two Markov chains:
\begin{equation}
\label{eq:chain}
    \begin{aligned}
        &x+\delta \rightarrow F(x+\delta), \\
        &z \rightarrow x_{re} \rightarrow F(x_{re}),
    \end{aligned}
\end{equation}
which implies that $F(x+\delta)$ is an indirect observation of $x+\delta$, whereas both $F(x_{re})$ and $x_{re}$ are indirect observations of $z$.

By the data processing inequality (DPI), we have
\begin{equation}
    I(x+\delta, z) \geq I(F(x+\delta), z),
\end{equation}
and
\begin{equation}
    I(F(x+\delta), z) \geq I(F(x+\delta), x_{re}) \geq I(F(x+\delta), F(x_{re})).
\end{equation}
\end{proof}

Now, we define $p_e$ as the probability that the predicted label of $x+\delta$ by $F$ is not equal to that of $x_{re}$, i.e., $p_e= \mathbb{P}( F(x+\delta) \neq F(x_{re}))$. Intuitively, our autoencoder is trained to recover only the natural sample $x$ without any interference from $\delta$. Hence, a relatively large value of $p_e$ is expected. In the following, we establish the connection between $p_e$ and $I(x+\delta, z)$ with both lower and upper bounds, showing that minimizing $I(x+\delta, z)$ also encourages a large value of $p_e$.

\begin{proposition}
\label{prop:one}
Let $H(\cdot)$ denote the information entropy and $H_b(p_e)=-p_e \log_2 p_e - (1-p_e) \log_2 (1-p_e)$ be the binary entropy, we have:
\begin{equation}
    \label{eq:lowerbound}
    H(F(x+\delta)) - H_b (p_e) - p_e \log (|F(x+\delta)| - 1) \leq I(x+\delta, z),
\end{equation}
where $|F(x+\delta)|$ is the total number of categories.\footnote{For instance, for CIFAR-10, $|F(x+\delta)|=10$.}
\end{proposition}

\begin{proof}
    By the chain rule of MI, we have 
\begin{equation}
    I(F(x+\delta), F(x_{re})) = H(F(x+\delta)) - H(F(x+\delta) | F(x_{re})).
\end{equation}

By applying Fano's inequality~\cite{fano1949transmission,8683204}, we obtain:
\begin{equation}\label{eq:fano}
    \begin{aligned}
        H(F(x+\delta) | F(x_{re})) \leq H_b(p_e) + p_e \log (|F(x+\delta)| - 1).
    \end{aligned}
\end{equation}

Adding $I(F(x+\delta), F(x_{re}))$ to both sides of Eq.~\eqref{eq:fano}:
\begin{equation}
\label{eq:lowermi_de}
    \begin{aligned}
        H(F(x+\delta)) & - H_b(p_e) - p_e \log (|F(x+\delta)| - 1) \\ 
        & \leq I(F(x+\delta), F(x_{re})) \\
        & \leq I(x+\delta, z).
    \end{aligned}
\end{equation}
The last line of Eq.~\eqref{eq:lowermi_de} is by Lemma~\ref{lemma1}.
\end{proof}

Therefore, we obtain a lower bound of $I(x+\delta,z)$.
If we use CIFAR-10 ($|F(x+\delta)|=10$) and assume the predicted labels $F(x+\delta)$ follow a uniform distribution, we can visualize the lower bound as a function of $p_e$ as shown in Figure~\ref{fig:bounds}, from which we observe an obvious monotonic inverse relationship between $I(x+\delta)$ and $p_e$ in the range $p_e \in [0, 0.9]$.
In fact, we can also obtain an upper bound under the assumption that $I(F(x+\delta),x_{re}) \approx I(x+\delta, z)$, i.e., there is no information loss in the two Markov chains in Eq.~\eqref{eq:chain}.

\noindent
\begin{proposition}
\label{prop:two}
If $I(F(x+\delta),x_{re}) \approx I(x+\delta, z)$, we have:
    \begin{equation}
\label{eq:milowerbound}
\begin{aligned}
    I(x+\delta, z) \lesssim H(F(x+\delta)) - 2p_e,
\end{aligned}
\end{equation}
\end{proposition}
in which the notation ``$\lesssim$" refers to less than or similar to.

\begin{proof}
    By the Hellman-Raviv inequality~\cite{1054466,brown2009information}, we have:
    \begin{equation}
    \begin{aligned}
        2p_e & \leq H(F(x+\delta)|x_{re}) \\
         & = H(F(x+\delta)) - I(F(x+\delta),x_{re}) \\
         & \approx H(F(x+\delta)) - I(x+\delta, z).
    \end{aligned}
    \end{equation}
\end{proof}

Similar to the lower bound, the upper bound also indicates $I(x+\delta, z)$ is inversely proportional to $p_e$ as shown in Figure~\ref{fig:bounds}. 
In fact, apart from the above-mentioned lower and upper bounds, there exists an alternative and intuitive way to understand the mechanism of minimizing $I(x+\delta,z)$. For simplicity, let us assume the natural data $x$ and adversarial perturbations $\delta$ are independent\footnote{This assumption is mild for certain scenarios, such as when considering universal or image-agnostic perturbations~\cite{Moosavi-Dezfooli_2017_CVPR}.}, then:
\begin{equation}
    I(x+\delta, z) = I(x,z) + I(\delta, z).
\end{equation}

According to~\cite{vincent2008extracting}, minimizing the expected reconstruction error between natural sample $x$ and corrupted input $x+\delta$ amounts to maximizing a lower bound of the mutual information $I(x,z)$, even though $z$ is a function of the corrupted input. Therefore, by minimizing $I(x+\delta,z)$, the network is forced to minimize $I(\delta,z)$ (since $I(x,z)$ is maximized).
In other words, only the adversarial information about $\delta$ has been removed from $z$ when minimizing $I(x+\delta,z)$. 
This also explains the robustness of $z$.

\section{Experimental Evaluation}
\label{sec:evaluation}

\subsection{Experimental Setup}
We evaluate MIMIR on three datasets: ImageNet-1K~\cite{5206848}, Tiny-ImageNet~\cite{le2015tiny}, and CIFAR-10~\cite{krizhevsky2009learning}, with diverse and commonly used architectures with multiple scales: ViT~\cite{alexey2021vit}, ConViT~\cite{pmlr-v139-d-ascoli21a}, Swin Transformer~\cite{Liu2021swin}, and ConvNext~\cite{liu2022convnext}.
Details of datasets are provided in Appendix~\ref{sec:datasets}.
Hyperparameters of the decoder are included in Appendix~\ref{sec:decoder_hyp}.

\noindent
\textbf{Training Setup.}
We train models from scratch for all experiments.
Following MAE~\cite{He_2022_CVPR}, we do pre-training by MIMIR for 800 epochs.
Please note that we also compare our MIMIR + fine-tuning paradigm with the End2End paradigm.
The End2End paradigm refers to the supervised training of a model from scratch without self-supervised pre-training.
To compare between End2End and pre-training + fine-tuning, the common training schedule practice is pre-training 800 epochs + fine-tuning 100 epochs versus End2End training 300 epochs in the existing works~\cite{He_2022_CVPR,Xie_2022_CVPR,Liu_2023_CVPR}.

The number of warmup epochs is 40 for pre-training.
We use AdamW~\cite{loshchilov2018decoupled} as the optimizer for both pre-training and fine-tuning.
We apply the cosine decay as the learning rate scheduler. 
At pre-training, MIMIR uses the 1-step PGD to generate adversarial examples for all three datasets.
The perturbation budget is $\epsilon=8, \alpha=10$.
For the fine-tuning stage, we use the 10-step PGD AT with perturbation bound $\epsilon=8, \alpha=2$ for CIFAR-10 and Tiny-ImageNet.
For ImageNet-1K, the perturbation bound is $\epsilon=4$.
We use different steps of PGD or APGD to generate the adversarial perturbation for training.
Details on training hyperparameters are provided in Appendix~\ref{sec:appendices_hyperparameters}.

\begin{table}
    \centering
    \setlength\tabcolsep{4pt}
    \caption{Comparison between End2End AT and Pre-training (MIMIR) + Fine-tuning using ViT-S on CIFAR-10.}
    \label{tab:cifar_at}
    \begin{tabular}{llcccc}
        \toprule
          Training & Method & Natural & PGD & AutoAttack \\
         \midrule
         \multirow{6}{*}{End2End} & AT~\cite{DBLP:conf/iclr/MadryMSTV18} & 75.36 & 32.84 & 26.17 \\
          & Fast AT~\cite{Wong2020Fast} & 76.81 & 32.57 & 21.41\\
          & TRADES~\cite{zhang2019trades} & 74.96 & 32.12 & 24.90\\
          & MART~\cite{wang2019improving} & 72.42 & 24.47 & 23.45\\
          & Generalist~\cite{Wang2023Generalist} & 60.88 & 14.44 & 11.20\\
          & DBAT~\cite{levi2024dbat} & 68.32 & 18.83 & 5.25\\
          \midrule
          \multirow{6}{*}{\tabincell{c}{MIMIR} } & AT~\cite{DBLP:conf/iclr/MadryMSTV18} & 86.56\up{11.20} & 56.76\up{23.92} & 45.39\up{19.22}\\
          & Fast AT~\cite{Wong2020Fast} & 87.22\up{10.41} & 49.17\up{16.6} & 35.89\up{14.48} \\
          & TRADES~\cite{zhang2019trades}  & 88.19\up{13.23} & 56.42\up{24.3} & 51.70\up{26.8}\\
          & MART~\cite{wang2019improving}  & 80.55\up{8.13} & 50.81\up{26.34} & 39.92\up{16.47}\\
          & Generalist~\cite{Wang2023Generalist} & 88.81\up{27.93} & 37.85\up{23.41} & 33.67\up{22.47}\\
          & DBAT~\cite{levi2024dbat} & 88.56\up{20.24} & 41.08\up{22.25} & 24.59\up{19.34}\\
          \bottomrule
    \end{tabular}
\end{table}

\begin{table}[b]
\centering
\caption{Standard deviation of 3 runs on 3 datasets.}
\label{tab:stdandmean}
\begin{tabular}{ccccc}
\toprule
Performance & CIFAR-10 & Tiny-ImageNet & ImageNet-1K \\
\midrule
Natural & 86.89$\pm$0.04 & 63.61$\pm$0.20 & 71.58$\pm$0.14\\
PGD & 56.19$\pm$0.33 & 26.44$\pm$0.05 & 42.44$\pm$0.08\\
\bottomrule
\end{tabular}
\end{table}

\begin{table*}
\centering
\caption{Comparison with SOTA results on ImageNet-1K under $\epsilon=4/255$. The ``Adv. Steps'' refers to attack steps for generating adversarial examples for AT. The results of~\cite{wang2024revisiting} are evaluated using 20-step PGD (AutoAttack is designed as a more powerful alternative to PGD), which is marked as $\dag$ in the table. Although AdvXL only uses 20 epochs, it takes more time due to huge datasets for pre-training and fine-tuning.}
\label{tab:allonimagenet}
\begin{tabular}{cccccccccccccc}
\toprule
Architecture & Params (M) & FT & FT Epoch & Adv. Steps & Source & Natural & AutoAttack\\
\midrule
DeiT-S & 22.1 & PGD & 100 & 1 & Augmentation warm-up~\cite{NEURIPS2021_e19347e1} & 66.62 & 36.56\\
DeiT-S & 22.1 & PGD & 110 & 1 & Light Recipe~\cite{10136149} & 66.80 & 37.90\\
ViT-S & 22.1 & PGD & 120 & 3 & EasyRobust~\cite{mao2022easyrobust} & 66.43 & 39.20\\
ViT-S & 22.1 & PGD & 300 & 3 & Adversarially Trained~\cite{liu2024comprehensive} & 70.7 & 43.7\\
RobArch-S & 26.1 & PGD & 110 & 3 & RobArch~\cite{peng2023robarch} & 70.17 & 44.14\\
ViT-S & 22.1 & APGD & 300 & 2 & Pre-training+AT~\cite{NEURIPS2023_2d3b0076} & 69.22 & 44.04\\
\rowcolor{bblue}
ViT-S & 22.1 & PGD & 300 & 3 & MIMIR & \textbf{71.52} & 45.90\\
\rowcolor{bblue}
ViT-S & 22.1 & APGD & 300 & 2 & MIMIR & 71.00 & 46.10\\
\rowcolor{bblue}
ViT-S & 22.1 & APGD & 300 & 3 & MIMIR & 70.96 & \textbf{46.16}\\
\midrule
ViT-B & 86.6 & ARD+PRM & 10 & 5 & ARD+PRM~\cite{NEURIPS2022_760b5def} & 69.10 &  34.62 \\
Swin-B & 87.7 & ARD+PRM & 10 & 5 & ARD+PRM~\cite{NEURIPS2022_760b5def} & 74.36 &  38.61 \\
ViT-B & 86.6 & PGD & 120 & 3 & EasyRobust~\cite{mao2022easyrobust} & 70.64 & 43.04\\
Swin-B & 87.7 & PGD & 120 & 3 & EasyRobust~\cite{mao2022easyrobust} & 75.05 & 47.42\\
RobArch-L & 104 & PGD & 100 & 3 & RobArch~\cite{peng2023robarch} & 73.44 & 48.94\\
ViT-B & 86.6 & PGD & 300 & 3 & Adversarially Trained~\cite{liu2024comprehensive} & 74.7 & 49.7\\
ViT-B & 86.6 & PGD & 20 & 3 & AdvXL~\cite{wang2024revisiting} & 73.4 & 53.0$\dag$ \\
ViT-B & 86.6 & APGD & 300 & 2 & Pre-training+AT~\cite{NEURIPS2023_2d3b0076} & 74.10 & 50.30\\ \rowcolor{bblue}
ViT-B & 86.6 & APGD & 100 & 2 & MIMIR & 74.40 & 51.92 \\
\rowcolor{bblue}
ViT-B & 86.6 & PGD & 100 & 3 & MIMIR & 75.68 & 52.96 \\
\rowcolor{bblue}
ViT-B & 86.6 & PGD & 300 & 3 & MIMIR & \textbf{76.98} & 53.84 \\
\rowcolor{bblue}
ViT-B & 86.6 & APGD & 300 & 2 & MIMIR & 76.32 & \textbf{54.28}\\
\bottomrule
\end{tabular}
\end{table*}

\noindent
\textbf{Evaluation Metrics.}
We use \textbf{natural accuracy} and \textbf{robust accuracy} as evaluation metrics.
Natural accuracy refers to the accuracy of natural and unmodified inputs.
The robust accuracy measures the accuracy under the AutoAttack (AA)~\cite{pmlr-v119-croce20b}.
AutoAttack is an ensemble of diverse parameter-free attacks, including white-box and black-box attacks.
In our experiments, we use the standard version of AutoAttack that contains four attacks, including APGD-ce~\cite{pmlr-v119-croce20b}, APGD-t~\cite{pmlr-v119-croce20b}, FAB-t~\cite{pmlr-v119-croce20a}, and Square~\cite{10.1007/978-3-030-58592-1_29}.
The perturbation budgets for evaluation are $\epsilon=8$ for CIFAR-10 and Tiny-ImageNet, $\epsilon=4$ for ImageNet-1K.
In addition, we also evaluate the MIMIR-trained models with unforeseen attacks, such as CW attacks, attacks with $l_2$ norm, and out-of-distribution data (ImageNet-Corruption~\cite{hendrycks2019robustness}).

\noindent
\textbf{Training stability.}
We also show that the natural accuracy and robustness of MIMIR are stable evaluation metrics.
Due to the high computation cost, we cannot report the standard deviation for all experiments.
To show that our method MIMIR has low variances, we train ViT-S on three datasets three times (1-step PGD AT for ImageNet-1K) and report the standard deviation and average performance in Table~\ref{tab:stdandmean}.

\subsection{Main Results}
\label{sec:ex_results}

We first explore different AT methods and MIMIR for ViT on CIFAR-10, demonstrating the fundamental incompatibility between conventional AT approaches and ViT architectures. 
Following this baseline evaluation, we scale our investigation to the more challenging ImageNet-1K dataset, demonstrating the generalizability and scalability of our proposed MIMIR framework. 
The subsequent sections present comprehensive experimental results across three benchmark datasets: CIFAR-10, Tiny-ImageNet, and ImageNet-1K. 
This multi-scale evaluation strategy allows a thorough analysis of MIMIR's effectiveness under varying conditions, from smaller to large-scale visual recognition tasks.

In addition, we also explore the effectiveness of elucidating diffusion model (EDM) data on AT in Appendix~\ref{sec:edm}.
This generated data is commonly used in AT to improve robustness~\cite{pmlr-v202-wang23ad,peng2023robust,rebuffi2021data,NEURIPS2021_21ca6d0c}.
Specifically, we use 5 million generated CIFAR-10 data and 1 million Tiny-ImageNet data provided by~\cite{pmlr-v202-wang23ad}.
Table~\ref{tab:perf_cifar10_edm} in Appendix~\ref{sec:edm} shows that EDM data significantly improves the robustness.

\noindent \textbf{CIFAR-10.}
Table~\ref{tab:cifar_at} shows the performance of End2End adversarial training from scratch and Pre-training (MIMIR) + Fine-tuning on ViT-S trained on CIFAR-10.
We provide the performance of 6 established or SOTA AT methods on CIFAR-10, indicating that traditional AT training strategies are not applicable to ViTs. 
Importantly, our experimental results also demonstrate that MIMIR can substantially improve all AT methods. 
The reason is that training ViTs from scratch is known to be difficult~\cite{alexey2021vit,zhu2023understanding} and even more difficult for adversarial training~\cite{NEURIPS2022_760b5def}.
For example, robust accuracy is lower than 30\% on ViT-B without pre-training~\cite{NEURIPS2022_760b5def}.
In contrast, MIMIR provides a more straightforward methodology and avoids this difficulty by switching to pre-training with a theoretically grounded MIM learning task.

\begin{table}[t]
\centering
\setlength\tabcolsep{4pt}
\caption{Time consumption of End2End vs. Pre-training+Fine-tuning.}
\label{tab:time_end2endvsptft}
\begin{tabular}{ccccc}
\toprule
Arch & GPU & AT Method & Epochs & Hours \\
\midrule
\multirow{2}{*}{ViT-S} & 4 A6000 & PGD$_{10}$ & 300 & 187.64\\
~ & 4 A6000 & MIMIR(PT)+PGD$_{10}$(FT) & 800(PT)+100(FT) & 123.76\\
\midrule
\multirow{2}{*}{ViT-B} & 4 A6000 & PGD$_{10}$ & 300 & 451.39\\
~ & 4 A6000 & MIMIR(PT)+PGD$_{10}$(FT) & 800(PT)+100(FT) & 263.77\\
\midrule
\multirow{2}{*}{Swin-L} & 4 H100 & PGD$_{3}$ & 300 & 363.42\\
~ & 4 H100 & MIMIR(PT)+PGD$_{3}$(FT) & 800(PT)+100(FT) & 231.14\\
\bottomrule
\end{tabular}
\end{table}

\begin{table*}[tb]
\centering
\caption{The performance of MIMIR on ImageNet-1K against unforeseen threats. The $l_2$ version of the CW~\cite{7958570} attack is limited by $c$. Higher $c$ allows a more powerful perturbation.}
\label{tab:unkonwnattack}
\begin{tabular}{cccccccccccc}
\toprule
Architecture & Method & Natural & CW ($l_2, c=1$) & CW ($l_\infty,\epsilon=4/255$) & AA ($l_2, \epsilon=2.0$) & AA ($l_\infty,\epsilon=4/255$) & ImageNet-C~\cite{hendrycks2019robustness}\\
\midrule
\multirow{2}{*}{ViT-S} & \cite{NEURIPS2023_2d3b0076} & 69.22 & 44.28 & 45.98 & 37.52 & 44.04 & 45.62\\
~ & MIMIR & \textbf{70.96} & \textbf{66.34} & \textbf{49.22} & \textbf{49.12}  & \textbf{46.16}  & \textbf{48.78} \\
\midrule
\multirow{2}{*}{ViT-B} & \cite{NEURIPS2023_2d3b0076} & 74.10 & 59.42 & 52.76 & 52.12 & 50.30 & 53.69\\
~ & MIMIR & \textbf{76.32} & \textbf{65.30} & \textbf{56.82} & \textbf{56.72} & \textbf{54.28} & \textbf{57.07}\\
\bottomrule
\end{tabular}
\end{table*}

\noindent
\textbf{Time Consumption (End2End vs. Pre-training(PT)+Fine-tuning(FT)).}
Note that we follow the standard way to compare End2End and MIMIR (Pre-training + Fine-tuning) training methods by fixing the training schedule, following existing works~\cite{He_2022_CVPR,Xie_2022_CVPR,Liu_2023_CVPR,Woo2023convnextv2,tian2023spark}, where we include pre-training 800 epochs + fine-tuning 100 epochs versus supervised End2End training of 300 epochs.
The reason for having a larger number of pre-training epochs is that self-supervised pre-training is much more efficient (see Table~\ref{tab:efficiency} in Appendix~\ref{sec:efficiency}) than End2End training, and the pre-trained backbone can be used multiple times for various fine-tuning tasks. 
For example, in Table~\ref{tab:cifar_at}, the 6 End2End AT methods cost $300\times6=1800$ fine-tuning epochs. 
MIMIR costs 800 pre-training epochs + $100\times6$ fine-tuning epochs, i.e., we only conduct the pre-training once for results in Table~\ref{tab:cifar_at}.
MIMIR pre-training epoch is more efficient than a fine-tuning epoch by discarding 75\% of image patches.

More specifically, Table~\ref{tab:time_end2endvsptft} shows the total time consumption of End2End vs. Pre-training+Fine-tuning on three architectures with ImageNet-1K.
Although MIMIR takes more training epochs, its pre-training+fine-tuning paradigm still costs less time than End2End adversarial training and gains much better performance on both natural and adversarial examples.
Additional time consumption results with different datasets can be found in Table~\ref{tab:efficiency} in Appendix~\ref{sec:efficiency}.

\noindent  \textbf{ImageNet-1K.} 
Table~\ref{tab:allonimagenet} compares MIMIR with previous works concerning adversarial robustness on ImageNet-1K ($l_\infty, \epsilon=4/255$), which follows the evaluation of common standardized RobustBench~\cite{croce2021robustbench}. 
Similar to other works~\cite{10136149,liu2024comprehensive,NEURIPS2023_2d3b0076}, we consider simpler AT methods (i.e., PGD and APGD AT) instead of the latest AT methods, such as Generalist~\cite{Wang2023Generalist} and DBAT~\cite{levi2024dbat}.
Indeed, since the latest methods introduce tailored components for CNNs to improve their adversarial robustness, they might not be effective for ViTs.
The number of parameters, training epochs, steps in the inner maximization of AT, and clean and robust accuracy are reported to provide a more detailed understanding of the performance. 
The robust accuracy is evaluated by AutoAttack on the RobustBench~\cite{croce2021robustbench} validation set (5,000 images).
We divide the models into: \emph{small} ($\approx 22$M) and \emph{large} ($\approx 86$M) models, corresponding to ViT-S and ViT-B.
Experimental results demonstrate that MIMIR outperforms all previous works across various training setups.

\noindent
\textbf{MIMIR with Various Architectures.}
In Table~\ref{tab:variousarch}, we show that MIMIR can be applied to diverse architectures. 
Specifically, we use three representative options, including ViT+convolutional blocks (CVST)~\cite{NEURIPS2023_2d3b0076}, the latest CNN architecture (ConvNext~\cite{liu2022convnext}), and a hierarchical vision transformer (Swin~\cite{Liu2021swin}).
The ViT+CVST refers to using ConvStem~\cite{NEURIPS2021_ff1418e8} to replace the patch embedding in ViTs with a convolutional block.
The ViT+CVST shows improved robustness compared to pure ViT models according to experiments in~\cite{NEURIPS2023_2d3b0076}.

As CNN and hierarchical architecture cannot accept variable-length inputs, MIMIR is not directly compatible with ConvNext and Swin.
To adapt MIMIR to the hierarchical Swin Transformer, we implemented Masked Image Modeling using Group Window Attention~\cite{huang2022green}, which groups image patches within each local window of arbitrary size and performs masked self-attention in each group.
To apply MIMIR to ConvNext, we use SparK~\cite{tian2023spark} for CNN to handle irregular and randomly masked input images, which is achieved by sparse convolution.
MIMIR achieves better or comparable results compared to SOTA results on RobustBench~\cite{croce2021robustbench}.

\noindent
\textbf{MIMIR against Unforeseen Attacks.}
Except for adversaries who are limited by the adversarial budget, e.g., $l_{\infty} = 8 \text{ or } 4$ or from PGD-family in our main experiments, MIMIR also shows the potential to provide robustness against unforeseen attacks and naturally corrupted data (ImageNet-C~\cite{hendrycks2019robustness}).
Table~\ref{tab:unkonwnattack} demonstrates the robustness of MIMIR against practical unforeseen attacks, including non-PGD attacks ($l_{2}$ and $l_{\infty}$ CW attack~\cite{7958570}), $l_2$ AutoAttack, and ImageNet-C~\cite{hendrycks2019robustness}.
It is clear that MIMIR still performs well against these unforeseen attacks.
In particular, MIMIR achieved top accuracy compared to the results of the ImageNet-C Leaderboard on the RobustBench.

\subsection{Ablation Study and Further Analysis}
\label{sec:ablation}

\noindent
\textbf{Step by step ablation.}
Table~\ref{tab:training_ablation} provides an ablation study to verify the design choices of MIMIR.
The ablation uses 100 epochs of 1-step PGD (PGD$_1$) AT as the baseline.
Then, we apply end-to-end clean ImageNet-1K pre-training (weights available in \texttt{timm} library\footnote{\url{https://github.com/huggingface/pytorch-image-models/blob/main/timm/models/vision_transformer.py}}) as initialization of AT.
After that, we replace the clean pre-training with MAE, adv MAE, and MIMIR step by step.
The adv MAE refers to using adversarial examples but not using the MI $I(x+\delta, z)$ in the loss.
The pre-training schedule is 800 epochs.
We also use stronger adversarial fine-tuning for better performance, including 2-step PGD (PGD$_2$), APGD (APGD$_2$) FT, and a longer fine-tuning scheduler (300 epochs).
Our results indicate that MIMIR outperforms baselines and can be further improved under the long training schedule.

\begin{figure*}[t]
\centering
\includegraphics[width = 1\linewidth]{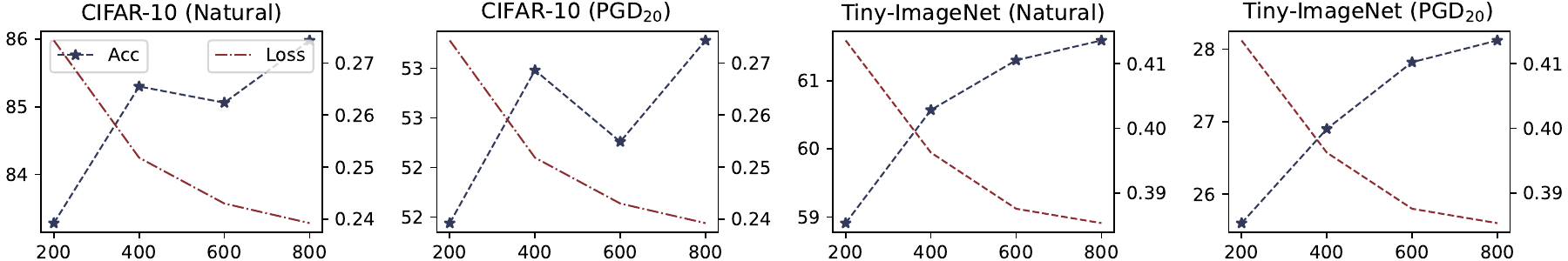}
\caption{Natural, adversarial accuracy, and MIMIR pre-training loss of ViT-S with different numbers of pre-training epochs under a 20-step PGD attack. The performance increases as the number of pre-training epochs increases (loss decreases).}
\label{fig:diff_epoch}
\end{figure*}

\begin{table}[t]
\centering
\caption{Comparsion with SOTA ImageNet-1K results on RobustBench~\cite{croce2021robustbench} with different architectures. 
$\dag$: The CVST modules are also pre-trained with MIMIR.}
\label{tab:variousarch}
\begin{tabular}{ccccccc}
\toprule
Architecture & Method & FT Epoch & Natural & AutoAttack\\
\midrule
\multirow{3}{*}{ViT-S+CVST} & \cite{NEURIPS2023_2d3b0076}
& 300 & 72.56 & 48.08\\
~ & MIMIR  & 300 & 72.72 & \textbf{48.44}\\
~ & MIMIR$\dag$  & 300 & \textbf{73.02} & 48.09\\
\midrule
\multirow{3}{*}{ViT-B+CVST} & \cite{NEURIPS2023_2d3b0076}
& 250 & 76.30 & 54.66\\
~ & MIMIR  & 300 & \textbf{76.72} & 54.04\\
~ & MIMIR$\dag$  & 300 & 76.32 & \textbf{55.08}\\
\midrule
 \multirow{2}{*}{ConvNext-T} & \cite{NEURIPS2023_2d3b0076} & 300 & 72.40 & 48.60 \\
 ~ & MIMIR  & 300 & \textbf{72.50} & \textbf{48.76}\\
\midrule
\multirow{2}{*}{Swin-B} & \cite{liu2024comprehensive}  & 300 & 76.16 & \textbf{56.16}\\
~ & MIMIR  & 150 & \textbf{76.62} & 55.90\\
\midrule
\multirow{2}{*}{Swin-L} & \cite{liu2024comprehensive}  & 300 & \textbf{78.92} & 59.56\\
~ & MIMIR  & 100 & 78.62 & \textbf{59.68}\\
\bottomrule
\end{tabular}
\end{table}

\noindent
\textbf{Longer epochs (lower loss) provide better performance. }
\textit{Pre-training epoch.}
Our experimental framework employs an 800-epoch pre-training as the baseline configuration. To systematically evaluate the impact of training duration, we conduct a comprehensive ablation study using ViT-S architectures, varying the number of pre-training epochs while maintaining a fixed 50-epoch fine-tuning for all models.
As demonstrated in Figure~\ref{fig:diff_epoch}, we observe several key phenomena.
Extended pre-training schedules consistently yield lower MIMIR loss values, and this reduction in loss results in measurable improvements in both adversarial and natural accuracy.
In addition, the improved performance with loss MIMIR loss also indicates that the model capacity is not saturated within the tested epoch range and could be further improved with a larger number of epochs.

\textit{Fine-tuning epoch.}
To isolate and quantify the contribution of MIMIR pre-training to model performance, we employ a short fine-tuning for the pre-trained models. 
This is to train the randomly initialized classification layer since we do not have the classification layer at pre-training.
This approach allows us to evaluate the quality of the representations learned during pre-training.
In Figure~\ref{fig:diff_ft_epoch}, we show that MIMIR pre-training plus 5 or 10 epochs of fine-tuning is enough to achieve similar performance compared to 100-epoch fine-tuning.
These results suggest that the majority of the model's final performance is attributable to the MIMIR pre-training phase.

\begin{table}[t]
\centering
\caption{Ablation of pre-training (PT) and fine-tuning (FT) methods on ImageNet-1K. ($\dag$: catastrophic over-fitting~\cite{Wong2020Fast} due to 1-step AT when fine-tuning, which can be fixed by 2-step AT. The fixed natural and robust accuracy are 69.96 and 36.90, respectively.)}
\label{tab:training_ablation}
\begin{tabular}{clccccc}
\toprule
Architecture & Training Recipe & Natural & AutoAttack\\
\midrule
\multirow{8}{*}{ViT-S} & PGD$_1$ FT w/o PT & 66.02 & 31.40\\
~ &  clean PT + PGD$_1$ FT & 67.04 & 33.70\\
~ & MAE PT + PGD$_1$ FT & 69.98 & 35.64\\
~ & adv MAE PT + PGD$_1$ FT & 68.24 & 19.32$\dag$\\
~ & MIMIR PT  + PGD$_1$ FT & 71.02 & 37.22\\
~ & MIMIR PT  + PGD$_2$ FT & 70.78 & 38.16\\
~ & MIMIR PT  + APGD$_2$ FT & 68.78 & 42.86\\
~ & 100 $\rightarrow$ 300 epochs of FT & 71.00 & 46.10\\
\bottomrule
\end{tabular}
\end{table}

\noindent
\textbf{MI measure.}
In Section~\ref{sec:theomotiv}, we provide lower and upper bound (Eq.~\eqref{eq:lowerbound}) of $I(x+\delta, z)$.
According to the two bounds, $I(x+\delta, z)$ is supposed to decrease while the autoencoder learns to reconstruct the natural image $x$.
This motivates us to directly embed $I(x+\delta, z)$ as a minimizing learning objective.
In this paper, we use $I_{\alpha}$~\cite{9414151} and HSIC~\cite{Ma_Lewis_Kleijn_2020} as estimators (detailed definitions in Appendix~\ref{sec:mi}).
Table~\ref{tab:lambda_hsic_mi} demonstrates the performance with different values of $\lambda$.
According to the results, we use HSIC with $\lambda=1e-05$ for all other experiments.

In addition, Figure~\ref{fig:mi_record} provides the quantities of HSIC values while pre-training with or without using MIMIR. It is clear that MIMIR can help to decrease the mutual information between adversarial perturbation and the learned features, i.e., $I(x+\delta, z)$.

\begin{figure}[t]
\centering
\includegraphics[width = 1\linewidth]{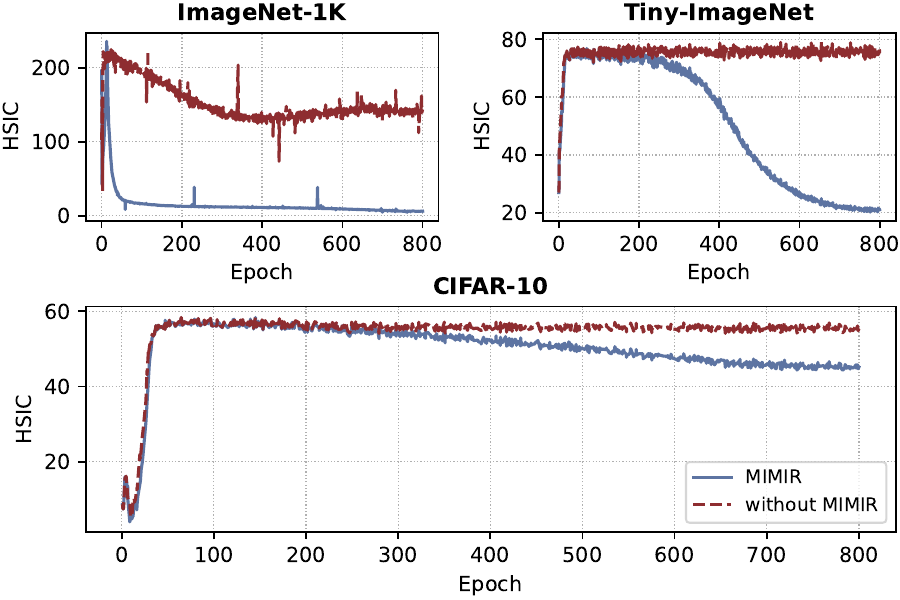}
\caption{The HSIC values (we use HSIC as an alternative to MI) while pre-training ViT-S with and without using MIMIR.}
\label{fig:mi_record}
\end{figure}

\begin{table}[t]
\centering
\caption{Comparison between HSIC~\cite{10.1007/11564089_7} and $I_{\alpha}$~\cite{9414151} using ViT-T on CIFAR-10. Models are pre-trained 800 epochs and adversarially fine-tuned with 10-step PGD for 50 epochs.}
\label{tab:lambda_hsic_mi}
\begin{tabular}{ccccc}
\toprule
Pre-train & $\lambda$ & Estimator & Natural & PGD\\
\midrule
MIMIR & 0.001 & HSIC & 69.63 & 43.17\\
MIMIR & 0.001 & $I_{\alpha}$ & 75.00 & 46.11\\
MIMIR & 1e-05 & HSIC & \textbf{76.30} & \textbf{47.60}\\
MIMIR & 1e-05 & $I_{\alpha}$ & 75.53 & 46.75\\
MIMIR & 1e-06 & HSIC & 74.90 & 46.19\\
MIMIR & 1e-06 & $I_{\alpha}$ & 74.60 & 45.66\\
\bottomrule
\end{tabular}
\end{table}

\noindent
\textbf{1-step is better than the 10-step of AT in pre-training.}
We also show that MIMIR outperforms original MAE~\cite{He_2022_CVPR} and adv MAE with different PGD steps (to generate adversarial examples for training).
MAE in Table~\ref{tab:hsic&mi} refers to using the original MAE for pre-training and then fine-tuning with 10-step PGD.
The adv MAE refers to using adversarial examples without the MI $I(x+\delta, z)$ in the loss.
The adv MAE (10-steps) refers to using the 10-step PGD algorithm ($\epsilon=8, \alpha=2$) to generate adversarial examples at pre-training.
The adv MAE provides higher accuracy than MAE, which supports our statement that using adversarial examples in Masked Image Modeling creates a more difficult reconstruction task.
This more difficult task further improves the performance of downstream models (see also Table~\ref{tab:plainfinetune}).
We use the default learning rate (i.e., $5.0e-4$) of MAE, so there is a performance drop in experiments in Tables~\ref{tab:lambda_hsic_mi} and~\ref{tab:hsic&mi} since AT prefers larger learning rates for CIFAR-10, as shown in Table~\ref{tab:diff_lr} in the appendix.

\begin{table}[t]
\centering
\caption{Comparison between different pre-training settings. 
All models are pre-trained for 800 epochs and then fine-tuned with 10-step PGD for 50 epochs using ViT-T on CIFAR-10.
}
\label{tab:hsic&mi}
\begin{tabular}{lcccc}
\toprule
Pre-train & $\lambda$ & Estimator & Natural & PGD\\
\midrule
MAE & 0.0 & - & 69.02 & 42.31\\
adv MAE (1-step) & 0.0 & - & 74.69 & 46.28\\
adv MAE (10-step) & 0.0 & - & 73.96 & 45.77 \\
MIMIR & 1e-05 & HSIC & \textbf{76.30} & \textbf{47.60}\\
\bottomrule
\end{tabular}
\end{table}

\begin{figure}[t]
\centering
\includegraphics[width = 1.0\linewidth]{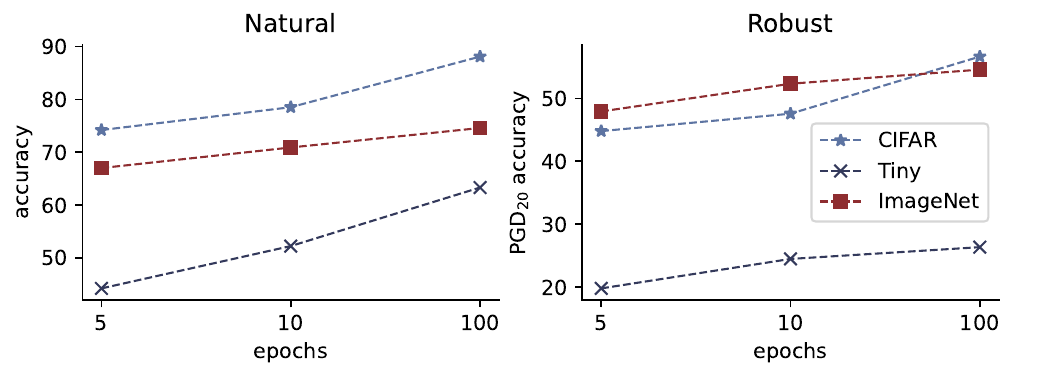}
\caption{Natural and adversarial accuracy of ViT-S adversarially fine-tuned for 5 or 10 epochs on CIFAR-10, Tiny-ImageNet, and ImageNet-1K.}
\label{fig:diff_ft_epoch}
\end{figure}

\noindent
\textbf{Data augmentation is not always harmful.}
Prior research~\cite{NEURIPS2021_e19347e1,10136149} has established that strong data augmentation techniques can adversely affect ViTs during adversarial training, as they may make training samples challenging to learn.
However, we observe that strong data augmentation does not impair model performance when combined with longer fine-tuning periods.

The strong data augmentation refers to the combination of Randaugment~\cite{Cubuk_2020_CVPR_Workshops}, CutMix~\cite{Yun_2019_ICCV}, and MixUp~\cite{DBLP:conf/iclr/ZhangCDL18}.
In this section, we evaluate two different solutions to ease this problem.
First, we only use simple data augmentation for adversarial training, including random crop (or random resize crop for ImageNet-1K) and random horizontal flip (``weak aug'').
Second, we use a 10-epoch warmup procedure for strong data augmentation.
The warmup of Randaugment is implemented by progressively increasing the distortion magnitude from 1 to 9 (``warmup aug'').
For CutMix and MixUp, we warm up by increasing the mixup probability from 0.5 to 1.0.
As shown in Figure~\ref{fig:three_augmentation}, ``weak aug'' provides the best accuracy.
The ``warmup aug'' shows a slightly improved accuracy compared to fusing strong augmentation.
Therefore, we provide a different result from~\cite{NEURIPS2021_e19347e1} on the smaller dataset CIFAR-10, i.e., we show that weak augmentation is better than warmup augmentation.
Even data augmentation with reduced amplitude is still difficult to learn at the beginning of adversarial training.
Although strong augmentation is harmful to a normal training schedule, 
we show in Table~\ref{tab:ablation_augmentation} that CutMix~\cite{Yun_2019_ICCV}, MixUp~\cite{DBLP:conf/iclr/ZhangCDL18}, and Randaugment~\cite{Cubuk_2020_CVPR_Workshops} increase the accuracy of adversarial training when training with a longer schedule, e.g., 800 epochs of fine-tuning. 
We conjecture that combining data and strong augmentations is helpful but difficult for adversarial training to learn.
Thus, more epochs are needed to learn meaningful representation.
Loss and accuracy curves can be found in Appendix~\ref{sec:appendices_ablation_augmentation}.

\begin{table}[t]
\centering
\setlength\tabcolsep{4pt}
\caption{Data augmentation with longer fine-tuning schedule.}
\label{tab:ablation_augmentation}
\begin{tabular}{cclcc}
\toprule
Arch & Epoch & Augmentation & Natural & PGD$_{20}$\\
\midrule
\multirow{3}{*}{ViT-B} & \multirow{3}{*}{800} & Weak Augmentation & 89.90 & 60.26 \\
~ & ~ & + CutMix~\cite{Yun_2019_ICCV},MixUp~\cite{DBLP:conf/iclr/ZhangCDL18} & 91.01 & 60.62\\
~ & ~ & + Randaugment~\cite{Cubuk_2020_CVPR_Workshops} & 90.19 & 62.75\\
\bottomrule
\end{tabular}
\end{table}

\begin{table}[t]
\centering
\setlength\tabcolsep{3.5pt}
\caption{Adversarial accuracy by adaptive attacks. The models are pre-trained for 800 epochs by MIMIR and fine-tuned for 100 epochs by 1-step PGD AT.}
\label{tab:adaptiveeval}
\begin{tabular}{ccccc}
\toprule
Dataset & Model &  PGD$_{20}$ & PGD-MI$_{100}$ & PGD-fea$_{100}$\\
\midrule
\multirow{3}{*}{CIFAR-10} & ConViT-S & 56.35 & 56.16 & 78.52\\
~ & ViT-S & 56.63 & 56.31 & 78.41\\
~ & ViT-B & 58.14 & 57.85 & 80.49\\
\midrule
\multirow{3}{*}{Tiny-ImageNet}  & ConViT-S & 26.39 & 26.29 & 58.50\\
~ & ViT-S & 26.37 & 26.18 & 57.36\\
~ & ViT-B & 25.41 & 25.05 & 58.90\\
\midrule
\multirow{3}{*}{ImageNet-1K}  & ConViT-S & 53.86 & 53.84 & 72.10 \\
~ & ViT-S & 54.56 & 54.55 & 72.27\\
~ & ViT-B & 55.41 & 55.36 & 73.51\\
\bottomrule
\end{tabular}
\end{table}

\subsection{Adaptive Attacks}
We evaluate MIMIR against adaptive adversaries following common practices~\cite{NEURIPS2020_11f38f8e}.
Adaptive adversaries possess the capability to devise targeted attacks specifically tailored to exploit the mechanisms of MIMIR, particularly if they have prior knowledge of its architecture and defensive strategies.
For example, the adversary may attack feature space~\cite{DBLP:journals/corr/SabourCFF15, 10.1145/3323873.3325052} since MIMIR trains the backbone to extract robust features.
Here, the backbone refers to the ViT model without the classification layer, i.e., the encoder of MIMIR.

We provide two adaptive attacks specifically designed against MIMIR.
First, we introduce the PGD Mutual Information attack (PGD-MI), which utilizes the MI $I(x+\delta, z)$ to generate adversarial examples, as $I(x+\delta, z)$ is used in MIMIR pre-training as a penalty in the loss.
PGD-MI attacks the model by directly increasing the MI $I(x+\delta, z)$.
Specifically, we add the MI loss into the PGD algorithm:
\begin{equation}
    \mathop{\max}\limits_{\delta \in S} \mathcal{L}_{CE}(x_i+\delta, y_i) + \lambda I(x+\delta, z).
\end{equation}
where the value of $\lambda$ in MIMIR pre-training is available to adversaries.
Second, we introduce a PGD feature attack (PGD-fea) that directly attacks the feature extracted by ViT backbones following~\cite{DBLP:journals/corr/SabourCFF15}.
In particular, we attack the feature extractor from the backbones after the adversarial fine-tuning.
The PGD-fea attack increases the Euclidean distance between features extracted from natural and adversarial examples.
We implement it using the PGD algorithm:
\begin{equation}
\begin{aligned}
    \mathop{\max}\limits_{\delta \in S}  \mathcal{L}_{\text{mse}}&(f_e(x), f_e(x+\delta)).
\end{aligned}
\end{equation}
Both PGD-MI and PGD-fea are optimized for 100 steps to ensure the attacking algorithm converges.
The perturbation budget is the same as the previous evaluation, i.e., $\epsilon=8/255$ for CIFAR-10 and Tiny-ImageNet, and $\epsilon=4/255$ for ImageNet-1K.

Table~\ref{tab:adaptiveeval} demonstrates the adaptive evaluation results for PGD-MI and PGD-fea attacks.
PGD-MI performs slightly better than the standard PGD attack, which means MI is exploitable information for perturbation crafting, but cannot significantly reduce the robustness.
Furthermore, MIMIR-trained models are converged to a local optimal where the majority of predictions are constantly around the ground truth within the ball function of $\epsilon$.
This also explains the resilience against both PGD and PGD-MI variants.
Regarding PGD-fea, it aims to maximize feature-space divergence rather than the distance of output logits.
However, MIMIR learned robust features that cannot be easily separated, so PGD-fea performs worse than PGD and PGD-MI.
The collective results demonstrate that MIMIR's IB framework induces robust learning dynamics that resist both output-space and feature-space attacks.

\begin{figure}[t]
\centering
\includegraphics[width = 1\linewidth]{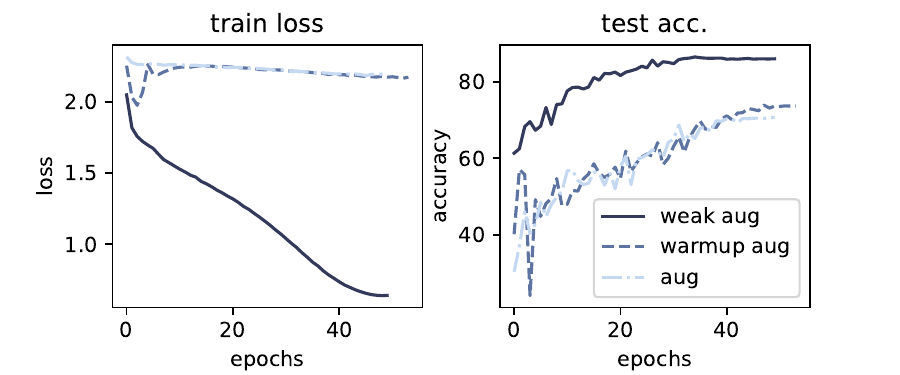}
\caption{Training loss and natural accuracy of ViT-S with three different data augmentations on CIFAR-10.}
\label{fig:three_augmentation}
\end{figure}

\begin{figure*}
\centering
\includegraphics[width = \linewidth]{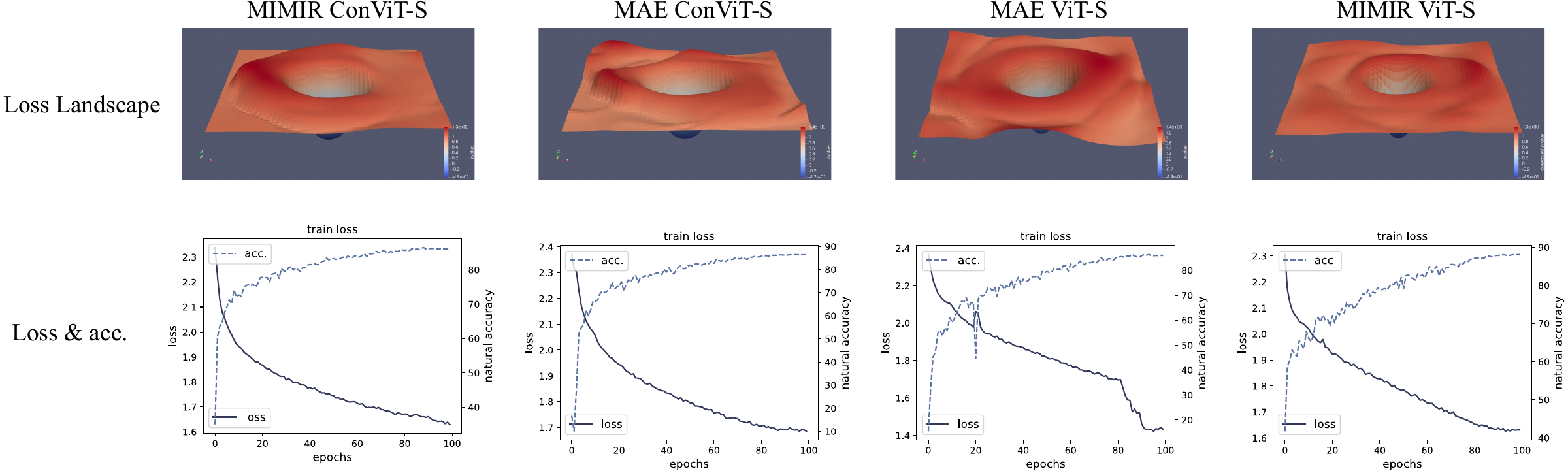}
\caption{The loss landscapes of MIMIR and MAE pre-trained models.}
\label{fig:loss_surface}
\end{figure*}

\begin{table}[t]
\centering
\caption{Natural accuracy of MAE and MIMIR (800 epochs pre-training for both) that are fine-tuned on natural images.}
\label{tab:plainfinetune}
\begin{tabular}{ccccc}
\toprule
Architecture & Pre-train & CIFAR-10 & Tiny-ImageNet & ImageNet-1K\\
\midrule
\multirow{2}{*}{ViT-B} & MAE & 96.79 & 73.38 & 82.92 \\
~ & MIMIR & \textbf{96.91} & \textbf{75.43} & \textbf{83.20}\\
\midrule
\multirow{2}{*}{ConViT-S} & MAE & 94.95 & 69.03 & 78.37 \\
~ & MIMIR & \textbf{95.38} & \textbf{70.40} & \textbf{79.21}\\
\midrule
\multirow{2}{*}{ViT-S} & MAE & 95.95 & 70.00 & 	77.45 \\
~ & MIMIR & 95.95 & \textbf{71.14} & \textbf{78.69}\\
\bottomrule
\end{tabular}
\end{table}

\subsection{Fine-tuning with Natural Images}
In Table~\ref{tab:plainfinetune}, we show the results of MIMIR when fine-tuning with natural images. We compared the performance with MAE~\cite{He_2022_CVPR}.
We fine-tune for 50 epochs for CIFAR-10 and Tiny-ImageNet, and 100 epochs for ImageNet-1K.
The results in Table~\ref{tab:plainfinetune} are reported with 800 pre-training epochs.
The base learning rate used in Table~\ref{tab:plainfinetune} is 0.001. 
The fine-tuning batch size is 512 for CIFAR-10 and Tiny-ImageNet, and 1024 for ImageNet-1K.
We use weak data augmentation (``weak aug''), which includes random crop and random horizontal flip.
Not surprisingly, MIMIR outperforms MAE when fine-tuning with natural data.
This is because MIMIR creates a harder learning task, which is helpful to learn more discriminative representation, as we discussed in Section~\ref{sec:designintuition}, Design Intuition.

According to Tables~\ref{tab:plainfinetune} and~\ref{tab:plainfinetune_robust}, MIMIR consistently shows improved performance on natural data.
Although the models in Table~\ref{tab:plainfinetune_robust} show poor robustness due to fine-tuning on natural data, MIMIR pre-trained ones provide slightly better robustness.
We want to clarify that poor robustness is expected when fine-tuning with natural data.
First, it is known that standard training on natural data learns non-robust features~\cite{NEURIPS2019_e2c420d9}, which hurts performance under adversarial attacks.
Second, MIMIR pre-training is implemented using MSE loss plus an MI penalty between natural inputs and adversarial images.
The adversarial perturbations and MI penalty help MIMIR create a more difficult and discriminative learning task to learn meaningful and robust features.
This process does not include the classification layer of the final model.
Therefore, MIMIR needs a fine-tuning process for superior performance on natural data and adversarial inputs.
In other words, the superior performance of our experiments comes from the combination of MIMIR and the simple fine-tuning process.

\subsection{Visualization of the Loss Landscape}
To show that the robustness of MIMIR-trained models does not stem from gradient masking, 
we plot the loss landscape~\cite{NEURIPS2018_a41b3bb3} in Figure~\ref{fig:loss_surface}.
The loss landscape is the visualization of the loss function as parameters change.
The basic idea is to plot the loss around the optimal parameters.
Formally, we consider in the 2D case,
\begin{equation}
    f_l(\alpha, \beta) = L(\theta^* + \alpha \theta_1 + \beta \theta_2),
\end{equation}
where $\theta_1 $ and $ \theta_2$ are two direction vectors, $\alpha$ and $\beta$ are two arguments of $f_l$.
In practice, we use the parameters of trained models, i.e.,  $\theta^*$. 
The landscapes of all models are smooth, i.e., the gradient at a certain point is clear and can also be easily estimated by local average gradients, which means the gradient is masked.

\begin{table}[t]
\centering
\caption{Natural and adversarial accuracy of ViT-S that is fine-tuned for 5 or 50 epochs with natural images.}
\label{tab:plainfinetune_robust}
\begin{tabular}{ccccccc}
\toprule
Fine-tune & Dataset & Pre-train & Natural & PGD\\
\midrule
\multirow{7}{*}{5 epochs} & \multirow{2}{*}{CIFAR-10} & MIMIR & 93.22 & 0.55\\
~ & ~ & MAE & 93.16 & 0.01\\
\cmidrule{2-5}
~ & \multirow{2}{*}{Tiny-ImageNet} & MIMIR & 63.65 & 0.00\\
~ & ~ & MAE & 62.21 & 0.00\\
\cmidrule{2-5}
~ & \multirow{2}{*}{ImageNet-1K} & MIMIR & 72.45 & 0.20\\
~ & ~ & MAE & 69.47 & 0.02\\
\midrule
\multirow{7}{*}{50 epochs} & \multirow{2}{*}{CIFAR-10} & MIMIR & 95.95 & 0.28\\
~ & ~ & MAE & 95.95 & 0.29\\
\cmidrule{2-5}
~ & \multirow{2}{*}{Tiny-ImageNet} & MIMIR & 71.14 & 0.00\\
~ & ~ & MAE & 70.00 & 0.00\\
\cmidrule{2-5}
~ & \multirow{2}{*}{ImageNet-1K} & MIMIR & 78.69 & 0.18\\
~ & ~ & MAE & 77.45 & 0.05\\
\bottomrule
\end{tabular}
\end{table}

\section{Related Work}
\label{sec:relatedwork}

\subsection{Vision Transformer}
The transformers~\cite{NIPS2017_3f5ee243} were first proposed in natural language processing (NLP).
With the mechanism of global self-attention, transformers can effectively capture the non-local relationships among all text tokens~\cite{DBLP:conf/naacl/DevlinCLT19,NEURIPS2020_1457c0d6,DBLP:conf/acl/DaiYYCLS19}.
A substantial effort is made to apply the transformer and self-attention mechanism in computer vision~\cite{alexey2021vit,Touvron_2021_ICCV,pmlr-v139-d-ascoli21a}.
The pioneering work, ViT~\cite{alexey2021vit}, demonstrated that the pure transformer architecture could achieve competitive performance on various tasks.
ViT also reveals that transformers lack inductive biases~\cite{alexey2021vit}.
For example, locality, two-dimensional neighborhood structure, and translation equivariance are inherent to CNNs but not applicable to ViTs~\cite{alexey2021vit}.
Due to this shortcoming, ViTs usually require large-scale training to get competitive performance, such as pre-training on ImageNet-21K~\cite{5206848} and JFT-300M~\cite{Sun_2017_ICCV}.
To alleviate the ViT need for large datasets, DeiT~\cite{pmlr-v139-touvron21a} introduced a teacher-student strategy to distill knowledge from a teacher CNN for a student ViT. 
In the MIM field, MAE~\cite{He_2022_CVPR} uses a masked autoencoder with a lightweight decoder as a visual representation learner.
Its learning objective is to reconstruct the original image by the decoder while using masked images as input to the autoencoder.
The advantage is that MAE can randomly discard 75\% image patches when pre-training under ImageNet-1K~\cite{5206848}, which means more efficient training.

\subsection{Adversarial Attacks on ViTs}
The concept of adversarial attacks first appeared in~\cite{10.1145/1014052.1014066}, which proposed a formal framework and algorithms against the adversarial spam detection domain.
Then, the adversarial attacks were popularized by Biggio et al.~\cite{10.1007/978-3-642-40994-3_25} and Szegedy et al.~\cite{szegedy2013intriguing} in image classification.
The generation of adversarial examples depends on the model's gradient or estimated gradient in a black box situation~\cite{pmlr-v80-ilyas18a}.
Therefore, adversarial attacks can be easily applied to transformers by using the gradient of attention blocks concerning inputs.
This also raises the question of whether transformers are more robust than CNNs. 
Benz et al.~\cite{DBLP:conf/bmvc/BenzHZKK21} found that CNNs are less robust than ViTs due to their shift-invariant property.
Bhojanapalli et al.~\cite{Bhojanapalli_2021_ICCV} found that ResNet models are more robust than transformers at the same model size under FGSM attack, but under PGD~\cite{DBLP:conf/iclr/MadryMSTV18} attack, transformer models show better robustness.
As ViTs process the input image as a sequence of patches, Gu et al.~\cite{10.1007/978-3-031-19775-8_24} found that ViTs are more robust than CNNs to naturally corrupted patches because the attention mechanism helps ignore naturally corrupted image patches.
The later work~\cite{NEURIPS2021_e19347e1} revealed that CNNs could be as robust as ViTs against adversarial attacks if CNNs are trained with proper hyperparameters.

\subsection{Adversarial Defense}
PGD~\cite{DBLP:conf/iclr/MadryMSTV18} adversarial training is considered one of the most effective defenses for CNNs and can withstand adaptively designed attacks~\cite{pmlr-v80-athalye18a}.
However, PGD AT is harmful to the accuracy of clean data~\cite{DBLP:conf/iclr/MadryMSTV18,pang2022robustness,levi2024dbat}.
Generalist~\cite{Wang2023Generalist} solves this problem by formulating different training strategies for robust and natural generalization separately.
DBAT~\cite{levi2024dbat} solves the decrease in natural accuracy by adding dummy classes~\cite{chen2018dummyclass} to the classification space.

Due to the difference between CNNs and ViTs, there have been some recent efforts to explore new adversarial training approaches for ViTs~\cite{NEURIPS2022_760b5def, 10136149,10.1007/978-3-031-19778-9_18}.
Mo et al.~\cite{NEURIPS2022_760b5def} presented a new adversarial training strategy based on the following observations: 1) pre-training with natural data can provide better robustness after adversarial fine-tuning, 2) gradient clipping is necessary for adversarial training, and 3) using SGD as the optimizer is better than Adam.
Debenedetti et al.~\cite{10136149} also presented an improved training strategy for ViTs by evaluating different combinations of data augmentation policies.
As adversarial training is time-consuming, AGAT~\cite{10.1007/978-3-031-19778-9_18} leverages the attention score while training to discard non-critical image patches after every layer.
Unlike previous works, we provide a different training paradigm by using MIM for adversarial pre-training.
Our method is efficient as we discard 75\% image patches while pre-training.
Our method is effective as we eliminate the information of adversarial perturbations from two information sources of natural and adversarial inputs.
We also provide theoretical proof that the information of adversarial perturbations is eliminated.

\subsection{Self-Supervised Adversarial Pre-Training}
Self-supervised learning~\cite{He_2020_CVPR,Chen_2021_CVPR,He_2022_CVPR,pmlr-v119-chen20j} refers to extracting meaningful representation from unlabeled data, which can be used for downstream recognition tasks.
Self-supervised methods are beneficial for out-of-distribution detection on difficult, near-distribution outliers~\cite{NEURIPS2019_a2b15837}, which leads to using self-supervised training to improve adversarial robustness~\cite{Chen_2020_CVPR,NEURIPS2019_a2b15837,NEURIPS2020_ba7e36c4,NEURIPS2021_b36ed8a0,wu2023denoising,rebuffi2023adversarially}.
The basic idea is to build a min-max learning object similar to traditional adversarial training.
For example, Jiang et al.~\cite{NEURIPS2020_ba7e36c4} considered using two adversarial samples or combining one adversarial sample and one natural sample to learn a consistent representation in contrastive learning.
In more recent work, You et al.~\cite{you2023beyond} proposed NIM De$^3$ to denoise adversarial perturbations.
However, the motivation of these works relies on complex self-supervised pre-training technologies, making it more difficult to understand the inner mechanisms or provide theoretical results.
MIMIR not only provides better performance but also provides intuitive insights with theoretical motivation.

\section{Discussion and Limitations}
\label{sec:discussion}

Following the principle of IB, we can intuitively consider a bottleneck between the encoder and decoder.
As the reconstruction output is constrained by natural data $x$, the bottleneck will filter out information from adversarial perturbations $\delta$.
We provide a theoretical guarantee of this bottleneck.
In Eq.~\eqref{eq:mi_loss}, we embed this bottleneck as a learning object to further improve the performance, which also confirms the correctness of our theoretical guarantee.
With the two information sources of $x$ and $\delta$, the model is trained to learn the robust features from $x$ and forget the information of $\delta$ under the constraint of the reconstruction target.

While MIMIR shows better performance, there are still limitations.
MIMIR is a pre-training method.
Adversarial fine-tuning is necessary to build the final robust model.
Thus, the shortcomings of traditional adversarial training cannot be completely avoided.
In our experiments, we utilize the simple PGD algorithm for fine-tuning, but one can further improve MIMIR pre-trained models with more advanced approaches.
In addition, MIMIR follows the design of MAE, and we also utilize the characteristic that ViTs can process variable-length inputs.
Therefore, the current MIMIR cannot directly handle CNNs.
While it is not trivial, we apply MIMIR to the latest CNN architecture by sparse convolution from SparK~\cite{tian2023spark}.
However, sparse convolution is not as efficient as dropping patch embeddings.
We leave these limitations to future work.

\section{Conclusions}
\label{sec:conclusions}

This paper provides a novel theoretical analysis of AT for ViTs through the lens of IB. 
We found that constraining the MI between adversarial perturbations and their latent representations in ViT-based autoencoders, as governed by derived MI bounds, is critical for enhancing model robustness.
Building upon this theoretical foundation, we propose MIMIR as a theoretically grounded pre-training method to improve adversarial robustness for ViTs.
MIMIR operates by processing adversarial examples as inputs while reconstructing their natural data as targets. 
This approach leverages the inherent information bottleneck in autoencoder architectures to achieve two key objectives: (1) progressively eliminating perturbation-related information while (2) preserving the essential features of the original data distribution. 
Our extensive experimental evaluation demonstrates that MIMIR significantly outperforms existing adversarial training methods across multiple benchmark datasets, achieving SOTA results on ImageNet-1K.
In addition, MIMIR is robust against unforeseen attacks and common corrupted data and can resist adaptive attacks.

\section*{Acknowledgment}
The authors thank the anonymous reviewers for their constructive comments that helped improve this paper and the artifact.
This work used the Dutch national e-infrastructure with the support of the
SURF Cooperative using grant no. EINF-10853. 
This work was partially supported by the Horizon Europe programme under the project SHASAI (No. 101225866).

\bibliographystyle{IEEEtran}
\bibliography{mimir_ref}

\clearpage

\appendix

\subsection{Datasets}
\label{sec:datasets}

We use three commonly used datasets to evaluate MIMIR: CIFAR-10~\cite{krizhevsky2009learning}, Tiny-ImageNet~\cite{le2015tiny}, and ImageNet-1K~\cite{5206848}.
CIFAR-10~\cite{krizhevsky2009learning} comprises 50,000 images with size $3 \times 32 \times 32$ in 10 classes.
ImageNet-1K~\cite{5206848} is the most commonly used dataset for the evaluation of ViTs and their variants, which is composed of more than 1.2 million high-resolution images in 1,000 classes. 
In our experiments, images from ImageNet-1K are resized to $3 \times 224 \times 224$.
For completeness, we also include Tiny-ImageNet~\cite{le2015tiny} as a medium size dataset between CIFAR-10~\cite{krizhevsky2009learning} and ImageNet-1K~\cite{5206848}.
Tiny-ImageNet~\cite{le2015tiny} contains 100,000 images with size $3 \times 64 \times 64$ in 200 classes.

\subsection{Decoder Hyperparameters}
\label{sec:decoder_hyp}

We use transformer blocks but fewer layers as the backbone of the decoder.
For CIFAR-10, we use the patch size of 2, 4 for Tiny-ImageNet, and 16 for ImageNet-1K.
Table~\ref{tab:model_parameter} shows the hyperparameters of decoder architectures.
For different ViT architectures, we use the transformer blocks of the respective architectures to build the encoder.

\begin{table}[t]
\centering
\caption{Model architectures of the encoder and decoder.}
\label{tab:model_parameter}
\begin{tabular}{ccccc}
\toprule
Model & Layers & Hidden size & MLP ratio & Heads \\
\midrule
ViT-T (encoder) & 12 & 192 & 4 & 3\\
ViT-S (encoder) & 12 & 384 & 4 & 6\\
ViT-B (encoder) & 12 & 768 & 4 & 12\\
decoder & 2 & 128 & 4 & 16\\
\bottomrule
\end{tabular}
\end{table}

\begin{table}
\footnotesize
\centering
\caption{Pre-training hyperparameters.}
\label{tab:pretraining_parameter}
\begin{tabular}{ll}
\toprule
Config & Value \\
\midrule
optimizer & AdamW \\
base learning rate & 1.5e-4 \\
weight decay & 0.05 \\
optimizer momentum & $\beta_1=0.9, \beta_2=0.95$\\
batch size & 512(CIFAR-10, Tiny), 2,048 (ImageNet-1K) \\
learning rate schedule & cosine decay \\
warmup epochs & 40 \\
training epochs & 800 \\
augmentation & RandomResizedCrop, RandomHorizontalFlip \\
\bottomrule
\end{tabular}
\end{table}

\subsection{Details of Training Hyperparameters}
\label{sec:appendices_hyperparameters}

In Tables~\ref{tab:pretraining_parameter} and~\ref{tab:finetune_parameter}, we provide the default hyperparameters used in our experiments.
We use different patch sizes for different datasets: patch size 2 for CIFAR-10, 4 for Tiny-ImageNet, and 16 for ImageNet-1K.
Using smaller patch sizes increases the time consumption when calculating self-attention, but MIMIR pre-training discards 75\% patches, making it still efficient.
Due to the depth and comparatively small embedding size of CaiT, we use a different drop path and layer-wise decay when fine-tuning (for ImageNet-1K).
For CaiT-XXS24, we use 0.95 and 0.15 as layer-wise decay and dropout, and 0.85 and 0.35 for CaiT-S36.
We also apply the stochastic depth decay rule~\cite{10.1007/978-3-319-46493-0_39} to CaiT.
CaiT-S36 models are only fine-tuned for 50 epochs due to time consumption, and it is sufficient to get superior results.
The batch size to fine-tune CaiT is 512 due to the limitation of GPU memory.
Other hyperparameters are consistent with Tables~\ref{tab:pretraining_parameter} and~\ref{tab:finetune_parameter}.

\begin{table}
\footnotesize
\centering
\caption{Fine-tuning hyperparameters.}
\label{tab:finetune_parameter}
\begin{tabular}{ll}
\toprule
Config & Value \\
\midrule
optimizer & AdamW \\
base learning rate & 0.5e-2 (CIFAR-10), 1e-3 (ImageNet, Tiny) \\
weight decay & 0.05 \\
optimizer momentum & $\beta_1=0.9, \beta_2=0.999$\\
layer-wise lr decay & 0.65 \\
batch size & 128 (CIFAR-10), 256 (Tiny), 1,024 (ImageNet) \\
learning rate schedule & cosine decay \\
warmup epochs & 10 \\
training epochs & 100 \\
augmentation & RandomResizedCrop, RandomHorizontalFlip \\
augmentation (IN1K) & CutMix, MixUp, Randaugmen\\
drop path & 0.1 \\
\bottomrule
\end{tabular}
\end{table}

\begin{table}[t]
\centering
\caption{Different learning rates. Fine-tuned for 50 epochs.}
\label{tab:diff_lr}
\begin{tabular}{cccccccccccccc}
\toprule
Dataset & Models & LR & Natural & PGD$_{10}$ \\
\midrule
\multirow{5}{*}{CIFAR-10} & \multirow{5}{*}{ViT-T}& 5.0e-4 & 76.30 & 47.60\\
~ & ~ &  1.0e-3 & 80.69 & 49.56\\
~ & ~ & 1.0e-2 & 85.62 & 48.78\\
~ & ~ & 5.0e-2 & 85.12 & 50.30\\
~ & ~ & 1.0e-1 & 84.51 & 50.40\\
\bottomrule
\end{tabular}
\end{table}

\subsection{Comparing MIMIR and MAE Performance with EDM}
\label{sec:edm}
In Table~\ref{tab:perf_cifar10_edm}, we compare the performance of MIMIR and MAE on CIFAR-10 and Tiny-ImageNet. Both methods pre-train for 800 epochs and fine-tune for 100 epochs.
MIMIR consistently outperforms MAE on both natural accuracy and adversarial robustness.
These results support our design intuition that adversarial noise builds a more difficult task for Masked Image Modeling, which helps the ViT encoder learn more discriminative features.

In addition, we use the elucidating diffusion model (EDM) data as data augmentation.
EDM generative data is usually used to improve the performance of adversarial training~\cite{pmlr-v202-wang23ad,peng2023robust,rebuffi2021data,NEURIPS2021_21ca6d0c}.
Specifically, we use 5 million generated CIFAR-10 data and 1 million Tiny-ImageNet data provided by~\cite{pmlr-v202-wang23ad}.
The EDM data is applied to experiments with CIFAR-10 and Tiny-ImageNet but not to ImageNet-1K (EDM data for ImageNet-1K are not provided in~\cite{pmlr-v202-wang23ad}).

\begin{table}
\centering
\caption{Natural and adversarial accuracy on CIFAR-10 and Tiny-Imagenet test set using MIMIR and MAE, pre-training (800 epochs) and then fine-tuning (100 epochs) using PGD adversarial training. We use EDM data from~\cite{pmlr-v202-wang23ad} as data augmentation.}
\label{tab:perf_cifar10_edm}
\begin{tabular}{cccccccccccc}
\toprule
Dataset & Arch & Pre-train & Natural & PGD$_{20}$ & AA\\
\midrule
\multirow{4}{*}{\tabincell{c}{CIFAR-10}} & \multirow{2}{*}{\tabincell{c}{ViT-S}} & MAE & 90.66 & 59.77 & 55.48\\
 ~ & ~ & MIMIR & \textbf{91.94} & \textbf{64.04} & \textbf{61.06}\\
 \cmidrule{2-6}
~ & \multirow{2}{*}{\tabincell{c}{ViT-B}} & MAE & 92.13 & 63.03 & 59.44\\
 ~ & ~ & MIMIR & \textbf{92.42} & \textbf{64.88} & \textbf{62.03}\\
 \midrule
 \multirow{4}{*}{\tabincell{c}{Tiny-ImageNet}} & \multirow{2}{*}{\tabincell{c}{ViT-S}} & MAE & 62.77 & 28.23 & 23.73\\
 ~ & ~ & MIMIR & \textbf{63.83} & \textbf{28.74} & \textbf{24.54}\\
  \cmidrule{2-6}
~ & \multirow{2}{*}{\tabincell{c}{ViT-B}} & MAE & 65.76 & 25.25 & 22.05\\
 ~ & ~ & MIMIR & \textbf{66.75} & \textbf{26.86} & \textbf{23.87}\\
\bottomrule
\end{tabular}
\end{table}

\subsection{Data Augmentation Evaluation}
\label{sec:appendices_ablation_augmentation}

Figure~\ref{fig:mix_longepoch} demonstrates the loss and accuracy while training with different augmentations. 
``no mix'' refers to using only weak augmentation, including RandomResizedCrop and RandomHorizontalFlip.
``+mix'' refers to using MixUp (0.8) and CutMix (1.0).
``+aug'' refers to using MixUp (0.8), CutMix (1.0), and Randaugment (rand-m9-mstd0.5-inc1).

\begin{figure}[ht]
\centering
\includegraphics[width = 1\linewidth]{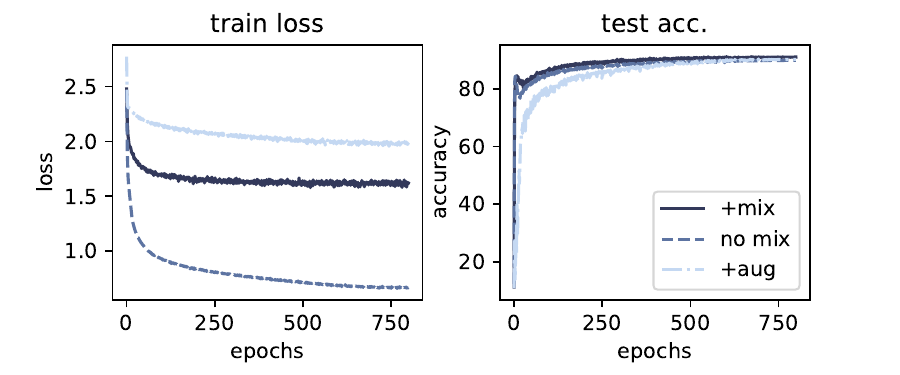}
\caption{The training results of using different data augmentations with 800 epochs.}
\label{fig:mix_longepoch}
\end{figure}

\subsection{Decoder Size}
Table~\ref{tab:decoder_size} provides the performance when pre-training with different decoder sizes. Specifically, the experiments are conducted on different numbers of decoder layers and different hidden sizes.
In Table~\ref{tab:decoder_size}, a deeper decoder may slightly increase the performance, but a larger hidden size may decrease the performance.
Additionally, increasing the size of the decoder will also increase the training cost, so we prefer to use a small decoder.

\begin{table}[]
    \centering
    \caption{The performance and pre-training time-consumption (hours) of ViT-S with different decoder sizes (depth and Hidden size).}
    \label{tab:decoder_size}
    \begin{tabular}{cccccc}
    \toprule
         Dataset & Layers & Hidden & Natural & PGD$_{20}$ & Time\\
         \midrule
         \multirow{4}{*}{\tabincell{c}{CIFAR-10}} & 2 & 256 & 86.45 & 55.12 & 5.28\\
         ~ & 2 & 512 & 85.45 & 51.57 & 6.76\\
         ~ & 4 & 128 & 86.52 & 55.10 & 5.08\\
         ~ & 6 & 128 & 87.37 & 56.93 & 6.01\\
         \midrule
         \multirow{4}{*}{\tabincell{c}{Tiny-ImageNet}} & 2 & 256 & 64.03 & 29.13 & 11.64\\
         ~ & 2 & 512 & 63.11 & 28.65 & 14.02\\
         ~ & 4 & 128 & 63.44 & 28.79 & 11.31\\
         ~ & 6 & 128 & 64.39 & 28.61 & 12.68\\
         \bottomrule
    \end{tabular}
\end{table}

\subsection{Subset of training data}

Table~\ref{tab:sub_train} shows the performance of pre-training on a small subset of training data, which explores whether MIMIR still performs well at a lower cost.
MIMIR performs well when it uses only 10\% of the training data and can achieve near-full data performance using only 25\% of the data.

\begin{table}[]
    \centering
    \caption{The performance of MIMIR pre-training on a subset of training data.}
    \label{tab:sub_train}
    \begin{tabular}{cccc}
    \toprule
        Dataset & Proportion & Natural & PGD$_{20}$\\
        \midrule
        \multirow{3}{*}{\tabincell{c}{CIFAR-10}} & 0.1 & 83.31 & 46.49\\
        ~ & 0.25 & 85.75 & 52.61\\
        ~ & 0.5 & 86.11 & 53.78\\
        \midrule
        \multirow{3}{*}{\tabincell{c}{Tiny-ImageNet}} & 0.1 & 57.92 & 24.61\\
        ~ & 0.25 & 60.85 & 26.46\\
        ~ & 0.5 & 62.44 & 28.13\\
\bottomrule
    \end{tabular}
\end{table}

\subsection{Layer-Wise MI}
Table~\ref{tab:layer_mi} shows the performance of using latent features from different layers to calculate the MI penalty in Equation~\ref{eq:mi_loss}.

We find a phenomenon of  MI oscillation, which occurs in layers closer to the inputs.
This is because the latent features in those layers do not include the additional normalization layer.
In the standard ViT design, an additional normalization layer is included after the final transformer layers to enhance stability and performance. 
In MIMIR, we use the latent features after the final normalization to calculate the MI penalty, i.e., layer 12 in the case of ViT-S. 
The MI oscillation also decreases the performance of fine-tuned models, as shown in Table~\ref{tab:layer_mi}.

\begin{table}[]
    \centering
    \caption{The performance of ViT-S while calculating MI for MIMIR at different layers.}
    \label{tab:layer_mi}
    \begin{tabular}{cccc}
    \toprule
        Dataset & Layer Index & Natural & PGD$_{20}$\\
        \midrule
        \multirow{4}{*}{\tabincell{c}{CIFAR-10}} & 5 & 86.45 & 54.51\\
        ~ & 7 & 85.93 & 53.22\\
        ~ & 9 & 86.39 & 53.41\\
        ~ & 12 & 86.56 & 56.76\\
        \midrule
        \multirow{4}{*}{\tabincell{c}{Tiny-ImageNet}} & 5 & 63.63 & 28.67\\
        ~ & 7 & 53.41 & 24.05\\
        ~ & 9 & 61.79 & 18.49\\
        ~ & 12 & 63.82 & 28.74\\
         \bottomrule
    \end{tabular}
\end{table}

\subsection{Efficiency}
\label{sec:efficiency}

We provide an analysis of the efficiency of MIMIR.
Table~\ref{tab:efficiency} provides the total time consumption and memory usage of different adversarial training methods, which are evaluated on four A6000 GPUs.
MIMIR is more efficient than 10-step PGD but slightly less efficient than FastAT, with higher robust accuracy than both 10-step PGD and FastAT.
Further, we provide the training time of MAE in Table~\ref{tab:efficiency}, which shows that the extra training time consumption introduced by the calculation of MI between $x+\delta$ and $z$ is small.

\begin{table*}
\centering
\setlength\tabcolsep{9pt}
\caption{The average time consumption on 4 GPUs. The ``mem.'' refers to GPU memory usage. The total time is estimated based on the time consumption on a single epoch. The training schedule for PGD$_{10}$ and FastAT is 300 epochs.
The training schedule for MAE and MIMIR is 800 epochs.}
\label{tab:efficiency}
\begin{tabular}{cccccccccccccc}
\toprule
\multicolumn{3}{c}{~} & \multicolumn{2}{c}{\textbf{CIFAR-10}~\cite{krizhevsky2009learning}} & \multicolumn{2}{c}{\textbf{Tiny-ImageNet}~\cite{le2015tiny}} & \multicolumn{2}{c}{\textbf{ImageNet-1K}~\cite{5206848}}\\
Architecture & \#Parames (M) & Method & time[H] & mem.[GB] & time[H] & mem.[GB] & time[H] & mem.[GB] \\
\midrule
\multirow{4}{*}{ViT-S} & \multirow{4}{*}{21.34} & PGD$_{10}$ AT & 12.44 & 2.54$\times$4 & 25.5 & 3.99$\times$4 & 187.64 & 12.5$\times$4\\
~ & ~ & FastAT & 3.61 & 2.54$\times$4 & 5.64 & 4.03$\times$4 & 46.29 & 10.4$\times$4\\
~ & ~ & MAE & 3.58 & 3.24$\times$4 & 7.33 & 3.27$\times$4 & 59.91 & 11.1$\times$4\\
~ & ~ & MIMIR & 4.09 & 3.12$\times$4 & 8.89 & 3.18$\times$4 & 61.22 & 11.1$\times$4\\
\midrule
\multirow{4}{*}{ViT-B} & \multirow{4}{*}{85.27} & PGD$_{10}$ AT & 30.1 & 5.39$\times$4 & 85.18 & 8.30$\times$4 & 451.39 & 22.1$\times$4\\
~ & ~ & FastAT & 10.23 & 5.36$\times$4 & 15.02 & 8.34$\times$4 & 113.44 & 19.8$\times$4\\
~ & ~ & MAE & 11.78 & 5.95$\times$4 & 23.67 & 5.95$\times$4 & 109.09 & 17.0$\times$4\\
~ & ~ & MIMIR & 13.11 & 6.08$\times$4 & 27.11 & 6.11$\times$4 & 113.31 & 17.0$\times$4\\
\midrule
\multirow{4}{*}{ConViT-S} & \multirow{4}{*}{27.05} & PGD$_{10}$ AT & 36.88 & 6.64$\times$4 & 74.75 & 12.19$\times$4 & 552.21 & 32.5$\times$4\\
~ & ~ & FastAT & 8.88 & 5.86$\times$4 & 15.27 & 10.62$\times$4 & 119.27 & 26.4$\times$4\\
~ & ~ & MAE & 7.33 & 10.6$\times$4 & 15.0 & 10.61$\times$4 & 135.49 & 27.5$\times$4\\
~ & ~ & MIMIR & 10.0 & 10.4$\times$4 & 20.0 & 10.54$\times$4 & 135.8 & 28.3$\times$4\\
\bottomrule
\end{tabular}
\end{table*}

\subsection{Mutual Information and HSIC}
\label{sec:mi}

MI measures the mutual dependence between two random variables, $X$ and $Y$. It can be decomposed as:
\begin{equation}
    \begin{aligned}
        I(X, Y) &= H(X) - H(X|Y),\\
        &=H(Y) - H(Y|X),\\
        &=H(X)+H(Y)-H(X,Y),
\end{aligned}
\end{equation}
where $H(X)$ and $H(Y)$ are the information entropies, $H(X|Y)$ and $H(Y|X)$ are the conditional entropies, and $H(X,Y)$ is the joint entropy of $X$ and $Y$.

Unfortunately, estimating MI in high-dimensional space is a difficult task since it may involve a precise estimation of the underlying data distribution $P_{(X,Y)}$ or $P_{(X)}$ and $P_{(Y)}$. 
To address this issue, the deterministic information bottleneck (DIB)~\cite{9414151} uses the recently proposed matrix-based R\'{e}nyi's $\alpha$-entropy functional $I_\alpha$~\cite{6954500,8787866}, which suggests similar quantities to $I(X,Y)$ in terms of the normalized eigenspectrum of the Hermitian matrix of the projected data in the reproducing kernel Hilbert space (RKHS), but avoids density estimation.

Specifically, given $N$ pairs of samples $(x_{i}, y_{i})_{i=1}^{N}$ (in our setup, $N$ refers to the mini-batch size), we can obtain two Gram (or kernel) matrices $K_x$ and $K_y$, for variables $X$ and $Y$, respectively, with
$(K_x)_{i,j}=\kappa_{x}(x_{i}, x_{j})$, $(K_y)_{i,j}=\kappa_{y}(y_{i}, y_{j})$, in which $\kappa_x$ and $\kappa_y$ are corresponding kernel functions. The information entropy of $X$ can be expressed as:
\begin{align}
\label{Renyi_entropy}
H_{\alpha}(X)&=\frac{1}{1-\alpha}\log_2 \left(\text{tr} (\tilde{K_x}^{\alpha})\right) \\ \nonumber
&=\frac{1}{1-\alpha}\log_{2}\left(\sum_{i=1}^{N}\lambda_i (\tilde{K_x})^{\alpha}\right),
\end{align}
where $\tilde{K}$ is the normalized version of $K$, i.e., $\tilde{K}=K/{\text{tr}(K)}$, and $\lambda _{i}(\tilde{K})$ denotes the $i$-th eigenvalue of $\tilde{K}$.

Further, the joint entropy for $X$ and $Y$ can be expressed as:
\begin{equation}\label{Renyi_joint_entropy}
H_{\alpha}(X,Y)=H_{\alpha}\left(\frac{K_x \circ K_y}{\text{tr}(K_x \circ K_y)}\right),
\end{equation}
where $K_x\circ K_y$  denotes the Hadamard product between the matrices $K_x$ and $K_y$.

Given Eqs.~\eqref{Renyi_entropy} and~\eqref{Renyi_joint_entropy}, the matrix-based R{\'e}nyi's $\alpha$-order mutual information $I_{\alpha}(X; Y)$ in analogy of Shannon's MI is given by:
\begin{equation}\label{Renyi_MI}
I_{\alpha}(X;Y)=H_{\alpha}(X)+H_{\alpha}(Y)-H_{\alpha}(X,Y).
\end{equation}
Throughout this paper, we use the radial basis function (RBF) kernel $\kappa(x_{i},x_{j})=\exp(-\frac{\|x_{i}-x_{j}\|^{2}}{2\sigma ^{2}})$ with kernel width $\sigma$ to obtain the Gram matrices. 

The Hilbert–Schmidt Independence Criterion (HSIC)~\cite{10.1007/11564089_7} is also a kernel-based dependence measure and is usually used as a surrogate of MI.
Formally, the HSIC is defined as the squared
norm of the cross-covariance operator $||C_{XY}||^2$:
\begin{equation}
\label{eq:estimator_hsic}
\begin{split}
	&\text{HSIC}_{P_{X,Y}}(X, Y) \\
	& = ||C_{XY}||^2\\
	& = \mathbb{E}_{xyx'y'}[\kappa_x(x,x')\kappa_{y'}(y,y')]\\
	& + \mathbb{E}_{xx'}[\kappa_x(x,x')]E_{yy'}[\kappa_y(y,y')]\\
	& -2\mathbb{E}_{xy}[\mathbb{E}_{x'}[\kappa_x(x,x')]\mathbb{E}_{y'}[\kappa_y(y,y')]],
\end{split}
\end{equation}
where $\kappa_x$ and $\kappa_y$ are kernel functions, $\mathbb{E}$ is the expectation, $x'$ and $y'$ are independent copies of $x$ and $y$, respectively.

Given $N$ pairs of samples $(x_{i}, y_{i})_{i=1}^{N}$, the empirical estimator of HSIC is given by:
\begin{equation}
    \widehat{\text{HSIC}}_{P_{X,Y}}(X, Y) = \frac{1}{N^2} \text{tr}\left( K_x H H_y H \right),
\end{equation}
in which $(K_x)_{i,j}=\kappa_{x}(x_{i}, x_{j})$, $(K_y)_{i,j}=\kappa_{y}(y_{i}, y_{j})$, and $H=I - \frac{1}{N} \mathbbm{1} \mathbbm{1}^T $ is the centering matrix.

\include{ndss_ae_appendix_mimir}

\end{document}

%% file: ndss_ae_appendix_mimir.tex


\appendices
\section*{Artifact Appendix}

\setcounter{section}{0}

\subsection{Description \& Requirements}
MIMIR is a self-supervised pre-training strategy for adversarial robustness of Vision Transformers (ViTs). The training process consists of a pre-training stage and a fine-tuning stage. This artifact supports the experiments and findings presented in the paper by providing the necessary code and trained weights. We also provide scripts for setting up a Python environment and running experiments.

\subsubsection{How to access}
The artifact is available at the permanent repository: \url{https://doi.org/10.5281/zenodo.17807275}

\subsubsection{Hardware dependencies}
Training and evaluation require at least one CUDA-enabled GPU, but we strongly recommend using more GPUs. In our case, experiments using small datasets (CIFAR-10, Tiny-ImageNet) are performed on two GPUs (RTX A6000 or RTX A5000). Experiments using the large dataset (ImageNet-1K) are performed on eight GPUs (RTX A6000, RTX A5000, or H100).

\subsubsection{Software dependencies}
The artifact is tested on Ubuntu 22.04.5 LTS with Python 3.10.12 and CUDA 12.7. The training script is implemented with PyTorch 2.1.0. All package dependencies are listed in \texttt{requirements.txt}.

\subsubsection{Benchmarks}
We use three commonly used benchmark datasets to evaluate the artifact: CIFAR-10, Tiny-ImageNet, and ImageNet-1K. CIFAR-10 and Tiny-ImageNet are included in the artifact. ImageNet-1k requires accepting the terms of access.\footnote{\url{https://image-net.org/download.php}}

\subsection{Artifact Installation \& Configuration}
To install necessary dependencies, ensure Python and CUDA are available. Then go to the root path and execute:
\begin{verbatim}
$ pip install -r requirements.txt
\end{verbatim}

\subsection{Experiment Workflow}
The artifact requires three primary workflows: (1) Pre-training with MIMIR. (2) Fine-tuning the MIMIR-trained models. (3) Evaluating the fine-tuned models.
The \texttt{bash} scripts corresponding to each workflow are provided in the artifact.

\subsection{Major Claims}
\begin{itemize}
    \item (C1): MIMIR is effective for ViTs on CIFAR-10. Adversarial training on ViTs is known to be difficult in previous works. This statement is supported by E1, and the results are shown in~\autoref{tab:cifar_at}.
    \item (C2): MIMIR is effective while scaling up to ImageNet-1K. This statement is supported by E2, and the results are shown in~\autoref{tab:allonimagenet}.
    \item (C3): MIMIR is effective against unforeseen attacks. This statement is supported by E3, and the results are shown in~\autoref{tab:unkonwnattack}.
    \item (C4): MIMIR is effective against adaptive attacks. This statement is supported by E4, and the results are shown in~\autoref{tab:adaptiveeval}.
\end{itemize}

There are also other experiments, but they may take more than several days to complete. We include corresponding scripts to execute them in the artifact, but not in this ``Major Claims''.

\subsection{Evaluation}
Overall, the experiments involve four steps for training: 
\begin{enumerate}
    \item Activate the Python environment. By default, we use virtual Python environments, which can be created by the following command lines:
    \begin{verbatim}
    $ python3 -m venv your_env
    $ source your_env/bin/activate 
    $ pip install -r requirements.txt
    \end{verbatim}
    \item Pre-training is implemented in \texttt{pretrain.py}.
    \item Fine-tuning is implemented in \texttt{finetune.py}.
    \item Evaluation is implemented in \texttt{finetune.py}, and can be activated by the flag \texttt{--eval}. Evaluation for ImageNet-1K with the 5000 RobustBench testset is implemented in \texttt{eval\_advmae/sub\_imagenet\_eval.py}.
\end{enumerate}

\begin{itemize}
    \item (E1): MIMIR training on CIFAR-10 with ViT-S.
    \begin{verbatim}
    $ cd scripts
    $ bash train_cifar10.sh
    \end{verbatim}
    \item (E2): MIMIR training on ImageNet-1K with ViT-S.
    \begin{verbatim}
    $ cd scripts
    $ bash train_imagenet.sh
    \end{verbatim}
    \item (E3): Evaluation against unforeseen attacks on ImageNet-1K.
    \begin{verbatim}
    $ cd scripts
    $ bash eval_subimagenet.sh
    $ bash eval_imagenet_c.sh
    \end{verbatim}
    \item (E4): Evaluation against adaptive attacks.
    \begin{verbatim}
    $ cd scripts
    $ bash eval_cifar10_adap.sh
    $ bash eval_tiny_adap.sh
    \end{verbatim}
\end{itemize}

